\documentclass{article}
\pdfoutput=1

\PassOptionsToPackage{numbers}{natbib}
\usepackage[nonatbib, preprint]{neurips_2020}

\usepackage[utf8]{inputenc} %
\usepackage[T1]{fontenc}    %
\usepackage{hyperref}       %
\usepackage{url}            %
\usepackage{booktabs}       %
\usepackage{amsfonts}       %
\usepackage{nicefrac}       %
\usepackage{microtype}      %

\usepackage{amsmath}
\usepackage{amssymb}
\usepackage{color}
\usepackage{amsthm}
\theoremstyle{plain}
\newtheorem{theorem}{Theorem}
\newtheorem{lemma}{Lemma}
\newtheorem*{theorem*}{Theorem}
\newtheorem{example}{Example}
\usepackage{subcaption}
\usepackage{graphicx}
\usepackage{wrapfig}
\usepackage[ruled,vlined]{algorithm2e}

\usepackage{tikz, ifthen}
\usetikzlibrary{angles, arrows, arrows.meta,calc,chains,shapes,fit,positioning,decorations.pathmorphing,quotes}
\usetikzlibrary{arrows.meta,calc,chains,shapes,fit,positioning,decorations.pathmorphing}
\def\layersep{3.2cm}
\usetikzlibrary{patterns, shapes, calc}

\usepackage{color}
\usepackage{colortbl}

\usetikzlibrary{arrows,positioning} 
\tikzset{
	>=stealth',
	punkt/.style={
		rectangle,
		rounded corners,
		draw=black, very thick,
		text width=6.5em,
		minimum height=2em,
		text centered},
	pil/.style={
		->,
		thick,
		shorten <=2pt,
		shorten >=2pt,}
	state/.style={circle,draw,minimum size=6ex},
	arrow/.style={-latex, shorten >=1ex, shorten <=1ex}
}

\usepackage{pgf}

\usepackage{subcaption,graphicx}
\newcommand{\rulesep}{\unskip\ \vrule\ }

\newcommand\pgfmathsinandcos[3]{%
	\pgfmathsetmacro#1{sin(#3)} 
	\pgfmathsetmacro#2{cos(#3)}}

\title{Reachable Sets of Classifiers and Regression Models: (Non-)Robustness Analysis and Robust Training}

\author{%
  Anna-Kathrin Kopetzki \\
  Department of Informatics \\
  Technical University of Munich \\
  Munich, Germany \\
  \texttt{kopetzki@in.tum.de} \\
  \And
  Stephan G\"unnemann \\
  Department of Informatics \\
  Technical University of Munich \\
  Munich, Germany \\
  \texttt{guennemann@in.tum.de} \\
}

\begin{document}

\maketitle

\begin{abstract}

Neural networks achieve outstanding accuracy in classification and regression tasks. 
However, understanding their behavior still remains an open challenge that requires questions to be addressed on the robustness, explainability and reliability of predictions. 
We answer these questions by computing \emph{reachable sets} of neural networks, i.e. sets of outputs resulting from continuous sets of inputs. 

We provide two efficient approaches that lead to over- and under-approxima\-tions of the reachable set. 
This principle is highly versatile, as we show. 
First, we use it to analyze and enhance the robustness properties of both classifiers and regression models. This is in contrast to existing works, which are mainly focused on classification. 
Specifically, we verify (non-)robustness, propose a robust training procedure, and show that our approach outperforms adversarial attacks as well as state-of-the-art methods of verifying classifiers for non-norm bound perturbations.
Second, we provide techniques to distinguish between reliable and non-reliable predictions for unlabeled inputs, to quantify the influence of each feature on a prediction, and compute a feature ranking.

\end{abstract}

\section{Introduction}
\label{sec:intro}

Neural networks are widely used in classification and regression tasks. However, understanding their behavior remains an open challenge and raises questions concerning the robustness, reliability and explainability of their predictions. We address these issues by studying the principle of reachable sets of neural networks: Given a \textit{set} of inputs, what is the \textit{set} of outputs of the neural network. 

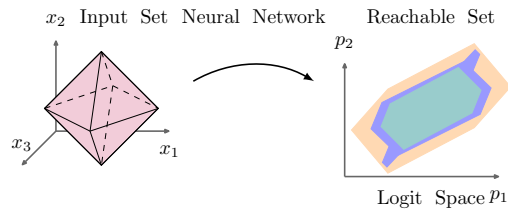
\begin{wrapfigure}{r}{0.5\textwidth}
	\centering
	\resizebox{0.5\textwidth}{!}{
		\begin{tikzpicture}[>={Stealth[inset=0pt,length=3.5pt,round]}, scale=0.4]

\coordinate (o1) at (0.5,-1.0);
\coordinate (x1) at (5.5,-1.0);
\coordinate (y1) at (0.5,3.0);
\coordinate (z1) at (-1.0, -2.5);

\coordinate (c1) at (3.5,2.0);
\coordinate (c1label) at (3.5,5.0);

\coordinate (gin1) at (0.6,1.52);
\coordinate (gin2) at (1.2,0.1);
\coordinate (gin3) at (0.8,-0.8);

\draw [->, thick, black!60] (o1) -- (x1) node[anchor=north] {$ $};
\draw [->, thick, black!60] (o1) -- (y1) node[anchor=north] {$ $};
\draw [->, thick, black!60] (o1) -- (z1) node[anchor=north] {$ $};

\coordinate (A1) at (0,0);
\coordinate (A2) at (3.0,1.0);
\coordinate (A3) at (5,0);
\coordinate (A4) at (2,-1);
\coordinate (B1) at (2.5,2.5);
\coordinate (B2) at (2.5,-2.5);

\filldraw[dashed, fill=purple!20] (A1) -- (A2) -- (A3);
\filldraw[fill=purple!20] (A1) -- (A4) -- (A3);
\filldraw[dashed, fill=purple!20] (B1) -- (A2) -- (B2);
\filldraw[fill=purple!20] (B1) -- (A4) -- (B2);
\filldraw[fill=purple!20] (B1) -- (A1) -- (B2) -- (A3) --cycle;
\draw[dashed] (A1) -- (A2) -- (A3);
\draw (A1) -- (A4) -- (A3);
\draw[dashed] (B1) -- (A2) -- (B2);
\draw (B1) -- (A4) -- (B2);
\draw (B1) -- (A1) -- (B2) -- (A3) --cycle;

\node[label=below:$\Huge{\mathrm{Input~Set}}$] (cnode) at (c1label) {}; %

\node[label=below:$\Huge{x_1}$] (cnode) at (x1) {}; %
\node[label=above:$\Huge{x_2}$] (cnode) at (y1) {}; %
\node[label=above:$\Huge{x_3}$] (cnode) at (z1) {}; %

\coordinate (c2) at (9.2,1.0);
\coordinate (c2label) at (9.2,5.0);

\coordinate (n1) at (-2.0,-2.0);
\coordinate (n2) at (-2.0,2.0);
\coordinate (n3) at (2.0,1.0);
\coordinate (n4) at (2.0,-1.0);

\coordinate (n5) at (3.0,0.0);

\draw [arrow, thick, bend left, bend angle=45]  ($(c2) - (n5)$) to ($(c2) + (n5)$);

\node[label=below:$\Huge{\mathrm{Neural~Network}}$] (cnode) at (c2label) {}; %

\coordinate (o3) at (13.2,-3.0);
\coordinate (o3label) at (17.0,-3.0);
\coordinate (x3) at (20.0,-3.0);
\coordinate (y3) at (13.2,2.0);

\coordinate (c3) at (17.0,0.0);
\coordinate (c3label) at (17.0,5.0);

\coordinate (go1) at (0.8,0.9);
\coordinate (go2) at (1.9,1.0);
\coordinate (go3) at (0.8,-0.95);

\coordinate (gu1) at (0.2,0.5);
\coordinate (gu2) at (1.5,0.8);
\coordinate (gu3) at (0.5,-0.5);

\coordinate (g1) at (0.24,0.58);
\coordinate (g2) at (1.8,0.9);
\coordinate (g3) at (0.6,-0.58);

\coordinate (h1) at (0.24,0.58);
\coordinate (h2) at (1.8,1.9);
\coordinate (h3) at (0.1,-0.1);

\filldraw [draw=orange!30,fill=orange!30]
($(c3) + (go1) + (go2) + (go3)$) --
($(c3) + (go1) + (go2) - (go3)$) --
($(c3) + (go1) - (go2) - (go3)$) --
($(c3) - (go1) - (go2) - (go3)$) --
($(c3) - (go1) - (go2) + (go3)$) --
($(c3) - (go1) + (go2) + (go3)$) --
cycle;
\filldraw [draw=blue!40,fill=blue!40]
($(c3) + (h1) + (h2) + (h3)$) --
($(c3) + (h1) + (h2) - (h3)$) --
($(c3) + (h1) - (h2) - (h3)$) --
($(c3) - (h1) - (h2) - (h3)$) --
($(c3) - (h1) - (h2) + (h3)$) --
($(c3) - (h1) + (h2) + (h3)$) --
cycle;
\filldraw [draw=blue!40,fill=blue!40]
($(c3) + (g1) + (g2) + (g3)$) --
($(c3) + (g1) + (g2) - (g3)$) --
($(c3) + (g1) - (g2) - (g3)$) --
($(c3) - (g1) - (g2) - (g3)$) --
($(c3) - (g1) - (g2) + (g3)$) --
($(c3) - (g1) + (g2) + (g3)$) --
cycle;
\filldraw [draw=teal!40,fill=teal!40]
($(c3) + (gu1) + (gu2) + (gu3)$) --
($(c3) + (gu1) + (gu2) - (gu3)$) --
($(c3) + (gu1) - (gu2) - (gu3)$) --
($(c3) - (gu1) - (gu2) - (gu3)$) --
($(c3) - (gu1) - (gu2) + (gu3)$) --
($(c3) - (gu1) + (gu2) + (gu3)$) --
cycle;

\node[label=below:$\Huge{\mathrm{Reachable~Set}}$] (cnode) at (c3label) {}; %

\draw [->, thick, black!60] (o3) -- (x3) node[anchor=north] {$ $};
\draw [->, thick, black!60] (o3) -- (y3) node[anchor=north] {$ $};
\node[label=below:$\Huge{p_1}$] (cnode) at (x3) {}; %
\node[label=above:$\Huge{p_2}$] (cnode) at (y3) {}; %
\node[label=below:$\Huge{\mathrm{Logit~Space}}$] (cnode) at (o3label) {}; %

\end{tikzpicture}
	}
	\caption{
		\textcolor{orange}{Over}-/\textcolor{teal}{Under}-approximation of the \textcolor{blue}{reachable set}.
	}
	\label{fig:z_reachset}
\end{wrapfigure}

Methods that compute an exact reachable set \cite{exactreach} are not feasible, even for tiny neural networks \cite{julia_toolbox}. 
In this study, we aim to approximate the reachable set such that it can be computed for neural networks used on standard data sets. More specifically, we investigate this problem in the context of ReLU neural networks, which constitute the most widely used class of networks.
To allow flexibility regarding inputs, we propagate a set of points defined by a zonotope through the neural network. As the ReLU operation can result in non-convex sets, we derive under-approximated or over-approximated output sets (see Figure \ref{fig:z_reachset}). The resulting sets are used to analyze and enhance neural network properties (see Section~\ref{sec:app_res}).

Overall, our main contributions are: 
(i) Two efficient approaches RsO and RsU (Reachable set Over- and Under-approximation) of approximating the reachable set of a neural network; 
(ii) Classification: Techniques of applying RsU and RsO to (non-)robustness verification, robust training, comparison with attacks, and state-of-the-art verification methods. 
(iii) Regression: an approach for analyzing and enhancing the robustness properties of regression models. 
(iv) Explainability/Reliability: a method of distinguishing between reliable and non-reliable predictions as well as a technique of quantifying and ranking features w.r.t.\ their influence on a prediction.

\section{Related Work}
\label{sec:relatedwork}

\textbf{Reachable Sets.} 
Computing the exact reachable set of a neural network as \cite{exactreach} is not applicable even with tiny networks, as shown in \cite{julia_toolbox}. 
Some techniques that approximate reachable sets, such as~\cite{ruan2018}, cannot handle the common robustness definition.
Most approaches that deal with the reachable sets of a neural network emerged from robustness verification.
The study that is the most closely related to our over-approximation approach RsO is~\cite{gehr2018}. 
Further developments of this technique include bounds \cite{deepzono, deeppoly, refinezono, singh2019krelu}. In addition, set-based approaches are used for robust training~\cite{mirman2018, gowal2018}. 
Our work goes beyond the existing approaches. 
First, our over-approximations are (by construction) subsets of the ones computed in~\cite{gehr2018} and thus tighter. 
Second, in comparison to the improvements presented in \cite{deepzono, deeppoly, refinezono}, our approaches do not require bounds on the layer input. %
Third, in addition to over-approximations, we provide an approach to under-approximate the reachable~set. 

\textbf{(Non-)Robustness Verification.} 
Reachable sets are applicable to (non-) robustness verification (see Section~\ref{sec:theo_intro}). 
Other expensive robustness verification methods are based on SMT solvers \cite{katz2017, ehlers2017, bunel2018}, mixed integer linear programming \cite{tjeng2019} or Lipschitz optimization \cite{ruan2018}. 
One family of verification approaches search for adversarial examples by solving the constrained optimization problem of finding a sample that is close to the input, but labeled differently. The search space, i.e. an over-approximation of the reachable set of the neural network is defined by the constraints. The distance of the samples to an input point is usually bound by a norm that the optimization problem can deal with, such as $L_{\infty}$-norm \cite{wong2017, raghunathan2018, bastani2016, katz2017, steinhardt2017} or $L_2$-norm \cite{hein2017}. 
One drawback of these approaches is the strong norm-based restriction on the inputs. 
Our approaches can handle input sets equivalent to norms as well as input sets that couple features and thus allow complex perturbations such as different brightness of pictures. 

The complement of robustness verification are adversarial attacks, i.e. points close to an input that are assigned a different label. Adversarial attacks compute a single point within the reachable set, without explicitly computing the reachable set. 
There are various ways of designing attacks, one of the strongest being the projected gradient descent attack (PGD) \cite{pgd_attack}. %
In contrast to attacks, our RsU approach aims to find an entire set of predictions corresponding to an input set. 

It should be noted that, all the above principles are designed for classification tasks. In contrast, our approach is naturally suited for regression as well. To further highlight the versatility of our method, we show how to apply it to explaining predictions and to distinguishing between reliable and non-reliable predictions.

\section{Reachable Sets of Neural Networks}
\label{sec:theo_intro}

The reachable set $O$ w.r.t.\ an input set $I$ of a neural network~$f$ is its output set, i.e. $O=\{f(x)\mid x\in I\}$. Computing the exact reachable set of a neural network is challenging, as proving simple properties of a neural network is already an NP-complete problem \cite{katz2017}. 
Under-approximations $\hat{O}_u \subseteq O$ produce a set of points that can definitely be reached with respect to the input, while over-approximations cover all points that might possibly be reached $O \subseteq \hat{O}_o$ (see Figure~\ref{fig:z_reachset}).

We propose approximating the reachable set by propagating the input set \textit{layer-by-layer} through the neural network.
In each layer, the input set is first subjected to the linear transformation defined by weights and biases. 
This linear transformation is computed \textit{exactly and efficiently} for the zonotope-based set representations we exploit. 
Then, the ReLU activation function is applied. 
Since applying ReLU onto a convex set can result in a non-convex set, we approximate convex subsets. Specifically, we propose an analytical solution for the over-approximations and an efficient linear optimization problem formulation for the under-approximations.

\textbf{Definition of Input Sets.}
Our approaches operate directly on sets and require an efficient and flexible set representation. 
For this, we use zonotopes, as they are closed under linear transformation and their G-representation provides a compact representation in high-dimensional spaces. 
Furthermore, they allow complex and realistic perturbations to be defined that couple input features such as different light conditions on pictures (in short: we go beyond simple and unrealistic norm constraints). 

The G-representation of a $d$-dimensional zonotope~$\hat{Z}$ with~$n$ generators is defined by a row-vector, the center~$\hat{c} \in \mathbb{R}^D$ and a matrix~$\hat{G} \in \mathbb{R}^{n \times D}$. The rows of this matrix contain the generators $\hat{g}_{i}$. The set of all points within~$\hat{Z}$ is:\vspace*{-2mm}
\begin{equation}
\begin{aligned}
	\hat{Z} = (\hat{c} \mid \hat{G}) := \left\{ \hat{c} + \sum_{i=1}^{n} \hat{\beta}_i \hat{g}_i \mid  \hat{\beta}_i \in [-1,1] \right\}\subset \mathbb{R}^D. \label{eq:zonotope}
\end{aligned}
\end{equation}

\textbf{Propagating Sets through ReLU Networks.}
In this paper we focus on ReLU neural networks, as they are not only widely used but also powerful.
A neural network consists of a series of functional transformations, in which each layer~$l$ (of $n_l$ neurons) receives the input~$x \in \mathbb{R}^{n_{l-1}}$ and produces the output by first subjecting the input to a linear transformation defined by the weights~$W^l$ and bias~$b^l$, and then %
applying $\mathrm{ReLU}$. %
In the final layer, no activation function is applied, and the output stays in the logit-space.
Thus, starting with the \emph{input set}~$\hat{Z}_0$ a series of alternating operations is obtained: 
$
	\hat{Z}^0 \overset{W^1,b^1}{\rightarrow} Z^1
	\overset{\mathrm{ReLU}}{\rightarrow} \hat{Z}^1
	\overset{W^2,b^2}{\rightarrow} Z^2
	\overset{\mathrm{ReLU}}{\rightarrow} \dots 
	\overset{W^L,b^L}{\rightarrow} Z^L,
$
where $Z^l$ denotes the set after the linear transformation, $\hat{Z}^l$ denotes the set after the ReLU, and $Z^L$ is the reachable set (output layer). 
Since zonotopes are closed under linear transformations, applying weights and bias of layer~$l$ to zonotope~$\hat{Z}^{l-1} = (\hat{c}^{l-1} \mid \hat{G}^{l-1})$ results in
\begin{equation}
\begin{aligned}
	Z^{l} = \left( c^{l} \mid G^{l} \right) = \left( W^l \hat{c}^{l-1} + b^l \mid \hat{G}^{l-1} W^{l T} \right). 
\end{aligned}
\label{eq:lintrans}
\end{equation}

Obtaining $\mathrm{ReLU} (Z^l)$ is challenging, as it may be a non-convex set, as illustrated in Figure~\ref{fig:relu_z}. 
It is inefficient to further propagate the non-convex set~$\mathrm{ReLU}(Z^l)$ through the neural network. 
Therefore, our core idea is to approximate~$\mathrm{ReLU} (Z^l)$ and use this as the input to the next layer. 
More precisely, we propose two methods: RsO (reachable set over-approximation) and RsU (reachable set under-approximation). 
RsO obtains a superset of ~$\mathrm{ReLU} (Z^l)$ in each layer~$l$, while RsU returns a subset. %
Using RsO within each layer ensures that no points are missed and that the output set captures all reachable points. 
Equivalently, applying RsU within each layer results in an output set that is a subset of the exact reachable set, i.e. contains the points that will definitely be reached. Pseudocode for RsO, RsU and propagating a zonotope through the neural network is provided in the appendix (see Section~\ref{sec:pseudocode}).

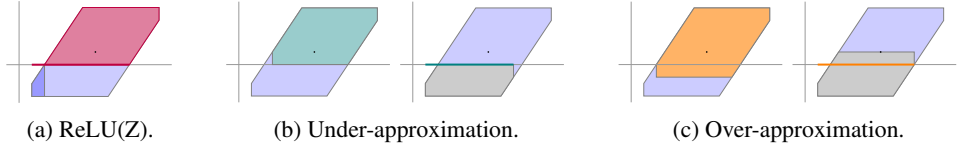
\begin{figure*}[t]
	\centering
	\begin{subfigure}{0.16087\textwidth}
		\resizebox{\textwidth}{!}{\begin{tikzpicture}[>={Stealth[inset=0pt,length=3.5pt,round]}, scale=0.76]

\coordinate (c) at (0,0);
\coordinate (g1) at (3.0,0.0);
\coordinate (g2) at (2.0,3.0);
\coordinate (g3) at (0.0,0.5);

\coordinate (g2p) at (1.0, 1.5);

\coordinate (e1) at ($(c) + (g1) + (g2) - (g3)$);
\coordinate (e2) at ($(c) + (g1) + (g2) + (g3)$);
\coordinate (e3) at ($(c) - (g1) + (g2) + (g3)$);
\coordinate (e4) at (-12/3, -1);
\coordinate (e5) at (8/3, -1);
\coordinate (e6) at (-5,-1);

\coordinate (x1) at(-7,-1);
\coordinate (x2) at(6,-1);
\coordinate (y1) at(-6,-4);
\coordinate (y2) at(-6,4);

\filldraw [draw=black!50,fill=blue!20]
($(c) + (g3) + (g2) + (g1)$) --
($(c) + (g3) + (g2) - (g1)$) --
($(c) + (g3) - (g2) - (g1)$) --
($(c) - (g3) - (g2) - (g1)$) --
($(c) - (g3) - (g2) + (g1)$) --
($(c) - (g3) + (g2) + (g1)$) --
cycle;

\draw [thick, black!40] (x1) -- (x2) node[anchor=north] {$ $};
\draw [thick, black!40] (y1) -- (y2) node[anchor=north] {$ $};

\coordinate (f1) at (-4.0, -3.5);
\coordinate (f2) at (0.0, -3.5);
\filldraw [draw=black!50,fill=blue!40]
(e4) --
(f1) --
($(c) - (g3) - (g2) - (g1)$) --
($(c) + (g3) - (g2) - (g1)$) --
cycle;

\filldraw [draw=purple!100, fill=purple!50]
(e1) --
(e2) --
(e3) --
(e4) --
(e5) --
cycle;

\draw [-, line width=1.2mm, purple!100] (e5) -- (e6);

\node[inner sep=1pt, circle, fill=black, label=below:$ $] (cnode) at (c) {}; %

\end{tikzpicture}}
		\caption{ReLU(Z).}
		\label{fig:relu_z}
	\end{subfigure}
	\hspace{0.5cm}
	\begin{subfigure}{0.325\textwidth}
		\resizebox{0.49\textwidth}{!}{\begin{tikzpicture}[>={Stealth[inset=0pt,length=3.5pt,round]}, scale=0.76]

\coordinate (c) at (0,0);
\coordinate (g1) at (3.0,0.0);
\coordinate (g2) at (2.0,3.0);
\coordinate (g3) at (0.0,0.5);

\coordinate (g2p) at (1.0, 1.5);

\coordinate (x1) at(-7,-1);
\coordinate (x2) at(6,-1);
\coordinate (y1) at(-6,-4);
\coordinate (y2) at(-6,4);

\coordinate (e1) at (-5.0, -3.5);
\coordinate (e2) at (-5.0, 3.5);
\coordinate (e3) at (5.0, 3.5);
\coordinate (e4) at (5.0, -3.5);

\coordinate (f1) at (-5.0, -1.0);
\coordinate (f2) at (-5.0, 3.5);
\coordinate (f3) at (5.0, 3.5);
\coordinate (f4) at (5.0, -1.0);

\filldraw [draw=black!50,fill=blue!20]
($(c) + (g3) + (g2) + (g1)$) --
($(c) + (g3) + (g2) - (g1)$) --
($(c) + (g3) - (g2) - (g1)$) --
($(c) - (g3) - (g2) - (g1)$) --
($(c) - (g3) - (g2) + (g1)$) --
($(c) - (g3) + (g2) + (g1)$) --
cycle;

\draw [thick, black!40] (x1) -- (x2) node[anchor=north] {$ $};
\draw [thick, black!40] (y1) -- (y2) node[anchor=north] {$ $};

\coordinate (cone) at (0.8+1/30, 1.25);
\coordinate (g1one) at (3.0,0.0);
\coordinate (g2one) at (1.1+2/30,1.75);
\coordinate (g3one) at (0.0,0.5);
\filldraw[draw=black!50,fill=teal!40]
($(cone) + (g3one) + (g2one) + (g1one)$) --
($(cone) + (g3one) + (g2one) - (g1one)$) --
($(cone) + (g3one) - (g2one) - (g1one)$) --
($(cone) - (g3one) - (g2one) - (g1one)$) --
($(cone) - (g3one) - (g2one) + (g1one)$) --
($(cone) - (g3one) + (g2one) + (g1one)$) --
cycle;

\node[inner sep=1pt, circle, fill=black, label=below:$ $] (cnode) at (c) {}; %

\coordinate (e5) at (8/3, -1);
\coordinate (e6) at (-5,-1);

\end{tikzpicture}}
		\resizebox{0.49\textwidth}{!}{\begin{tikzpicture}[>={Stealth[inset=0pt,length=3.5pt,round]}, scale=0.76]

\coordinate (c) at (0,0);
\coordinate (g1) at (3.0,0.0);
\coordinate (g2) at (2.0,3.0);
\coordinate (g3) at (0.0,0.5);

\coordinate (g2p) at (1.0, 1.5);

\coordinate (x1) at(-7,-1);
\coordinate (x2) at(6,-1);
\coordinate (y1) at(-6,-4);
\coordinate (y2) at(-6,4);

\coordinate (e1) at (-5.0, -3.5);
\coordinate (e2) at (-5.0, 3.5);
\coordinate (e3) at (5.0, 3.5);
\coordinate (e4) at (5.0, -3.5);

\coordinate (f1) at (-5.0, -1.0);
\coordinate (f2) at (-5.0, 3.5);
\coordinate (f3) at (5.0, 3.5);
\coordinate (f4) at (5.0, -1.0);

\filldraw [draw=black!50,fill=blue!20]
($(c) + (g3) + (g2) + (g1)$) --
($(c) + (g3) + (g2) - (g1)$) --
($(c) + (g3) - (g2) - (g1)$) --
($(c) - (g3) - (g2) - (g1)$) --
($(c) - (g3) - (g2) + (g1)$) --
($(c) - (g3) + (g2) + (g1)$) --
cycle;

\draw [thick, black!40] (x1) -- (x2) node[anchor=north] {$ $};
\draw [thick, black!40] (y1) -- (y2) node[anchor=north] {$ $};

\node[inner sep=1pt, circle, fill=black, label=below:$ $] (cnode) at (c) {}; %

\coordinate (ch) at (-1.5, -2.25);
\coordinate (h1) at (3.0, 0.0);
\coordinate (h2) at (0.5, 0.75);
\coordinate (h3) at (0.0, 0.5);
\filldraw [draw=black!50,fill=black!20]
($(ch) + (h3) + (h2) + (h1)$) --
($(ch) + (h3) + (h2) - (h1)$) --
($(ch) + (h3) - (h2) - (h1)$) --
($(ch) - (h3) - (h2) - (h1)$) --
($(ch) - (h3) - (h2) + (h1)$) --
($(ch) - (h3) + (h2) + (h1)$) --
cycle;

\coordinate (e5) at (2, -1);
\coordinate (e6) at (-5,-1);
\draw [-, line width=1.2mm, teal!100] (e5) -- (e6);

\end{tikzpicture}}
		\caption{Under-approximation.}
		\label{fig:under_z}
	\end{subfigure}
	\hspace{0.5cm}
	\begin{subfigure}{0.325\textwidth}
		\resizebox{0.49\textwidth}{!}{\begin{tikzpicture}[>={Stealth[inset=0pt,length=3.5pt,round]}, scale=0.76]

\coordinate (c) at (0,0);
\coordinate (g1) at (3.0,0.0);
\coordinate (g2) at (2.0,3.0);
\coordinate (g3) at (0.0,0.5);

\coordinate (g2p) at (1.0, 1.5);

\coordinate (x1) at(-7,-1);
\coordinate (x2) at(6,-1);
\coordinate (y1) at(-6,-4);
\coordinate (y2) at(-6,4);

\coordinate (e1) at (-5.0, -3.5);
\coordinate (e2) at (-5.0, 3.5);
\coordinate (e3) at (5.0, 3.5);
\coordinate (e4) at (5.0, -3.5);

\coordinate (f1) at (-5.0, -1.0);
\coordinate (f2) at (-5.0, 3.5);
\coordinate (f3) at (5.0, 3.5);
\coordinate (f4) at (5.0, -1.0);

\filldraw [draw=black!50,fill=blue!20]
($(c) + (g3) + (g2) + (g1)$) --
($(c) + (g3) + (g2) - (g1)$) --
($(c) + (g3) - (g2) - (g1)$) --
($(c) - (g3) - (g2) - (g1)$) --
($(c) - (g3) - (g2) + (g1)$) --
($(c) - (g3) + (g2) + (g1)$) --
cycle;

\coordinate (ch) at (0.5, 0.75);
\coordinate (h1) at (3.0,0.0);
\coordinate (h2) at (1.5,2.25);
\coordinate (h3) at (0.0,0.5);
\filldraw[draw=black!50,fill=orange!60]
($(ch) + (h3) + (h2) + (h1)$) --
($(ch) + (h3) + (h2) - (h1)$) --
($(ch) + (h3) - (h2) - (h1)$) --
($(ch) - (h3) - (h2) - (h1)$) --
($(ch) - (h3) - (h2) + (h1)$) --
($(ch) - (h3) + (h2) + (h1)$) --
cycle;

\draw [thick, black!40] (x1) -- (x2) node[anchor=north] {$ $};
\draw [thick, black!40] (y1) -- (y2) node[anchor=north] {$ $};

\node[inner sep=1pt, circle, fill=black, label=below:$ $] (cnode) at (c) {}; %

\end{tikzpicture}}
		\resizebox{0.49\textwidth}{!}{\begin{tikzpicture}[>={Stealth[inset=0pt,length=3.5pt,round]}, scale=0.76]

\coordinate (c) at (0,0);
\coordinate (g1) at (3.0,0.0);
\coordinate (g2) at (2.0,3.0);
\coordinate (g3) at (0.0,0.5);

\coordinate (g2p) at (1.0, 1.5);

\coordinate (x1) at(-7,-1);
\coordinate (x2) at(6,-1);
\coordinate (y1) at(-6,-4);
\coordinate (y2) at(-6,4);

\coordinate (e1) at (-5.0, -3.5);
\coordinate (e2) at (-5.0, 3.5);
\coordinate (e3) at (5.0, 3.5);
\coordinate (e4) at (5.0, -3.5);

\coordinate (f1) at (-5.0, -1.0);
\coordinate (f2) at (-5.0, 3.5);
\coordinate (f3) at (5.0, 3.5);
\coordinate (f4) at (5.0, -1.0);

\filldraw [draw=black!50,fill=blue!20]
($(c) + (g3) + (g2) + (g1)$) --
($(c) + (g3) + (g2) - (g1)$) --
($(c) + (g3) - (g2) - (g1)$) --
($(c) - (g3) - (g2) - (g1)$) --
($(c) - (g3) - (g2) + (g1)$) --
($(c) - (g3) + (g2) + (g1)$) --
cycle;

\draw [thick, black!40] (x1) -- (x2) node[anchor=north] {$ $};
\draw [thick, black!40] (y1) -- (y2) node[anchor=north] {$ $};

\coordinate (ch) at (-1.1666999999999996, -1.75);
\coordinate (h1) at (3.0, 0.0);
\coordinate (h2) at (0.83333333, 1.25);
\coordinate (h3) at (0.0, 0.5);
\filldraw [draw=black!50,fill=black!20]
($(ch) + (h3) + (h2) + (h1)$) --
($(ch) + (h3) + (h2) - (h1)$) --
($(ch) + (h3) - (h2) - (h1)$) --
($(ch) - (h3) - (h2) - (h1)$) --
($(ch) - (h3) - (h2) + (h1)$) --
($(ch) - (h3) + (h2) + (h1)$) --
cycle;

\node[inner sep=1pt, circle, fill=black, label=below:$ $] (cnode) at (c) {}; %

\coordinate (e5) at (8/3, -1);
\coordinate (e6) at (-5,-1);
\draw [-, line width=1.2mm, orange!100] (e5) -- (e6);

\end{tikzpicture}}
		\caption{Over-approximation.}
		\label{fig:over_z}
	\end{subfigure}
	\caption{Application of ReLU to a zonotope (blue) can result in a non-convex set (red). We approximate each subset located in each quadrant separately (here: two quadrants) and subject it to ReLU. The obtained set of sets under-approximates (green sets) or over-approximates (orange sets) $\mathrm{ReLU}(Z)$.}
	\label{fig:theory_approx}
\end{figure*}

\textbf{Approximation of ReLU(Z).}
In the following, we describe how to approximate ReLU($Z$) based on zonotope $Z$. To unclutter the notiation, we omit layer index~$l$. 
The ReLU function maps points dimension-wise onto the maximum of themselves and zero. %
Consideration of dimension~$d$ results in three possible cases:
	Case~$1$: $\forall p \in Z: p_d < 0$, where the points are mapped to zero, 
	Case~$2$: $\forall p \in Z: p_d \geq 0$, where the points are mapped onto themselves and 
	Case~$3$: $\exists p, q \in Z: p_d < 0 \wedge q_d > 0$, where the points are mapped to zero or themselves. 

Case~3 causes the non-convexity of ReLU (see Figure~\ref{fig:relu_z}, $2^{\mathrm{nd}}$ dimension). 
We consider the three cases separately to approximate each maximum convex subset of $\mathrm{ReLU} (Z)$ by one zonotope. 
The three cases are distinguished by computing an index set for each case:  
\begin{equation}
\begin{aligned}
   R_{\mathrm{n}} &= \{d \mid \forall p \in Z: p_d < 0 \}, \mathrm{~~~~}
&& R_{\mathrm{p}}  = \{d \mid \forall p \in Z: p_d \geq 0 \}, \\
   R              &= \{d \mid \exists p, q \in Z: p_d < 0, q_d > 0 \}
   \label{eq:index_sets}
\end{aligned}
\end{equation}
These index sets can be efficiently obtained through the interval hull of~$Z$ ~\cite{kuehn1998}, where $\left|.\right|$ is the element-wise absolute value:
$	\mathrm{IH}\left(Z\right) := \left[c - \delta g, c + \delta g \right]$ where $\delta g = \sum_{i} |g_i|$.
$R_{\mathrm{n}}$ contains the dimensions~$d$ such that $(c - \delta g)_d \leq  0$, $R_{\mathrm{p}}$ contains the dimensions where $(c + \delta g)_d \geq 0$, and $R$ contains the remaining dimensions. %

\textbf{Projection of a Zonotope.}
Regarding the dimensions in~$R_{\mathrm{n}}$, ReLU maps each point of the zonotope to zero.
Thus, we can safely project the whole zonotope~$Z = (c \mid G )$ to zero within these dimensions.

\begin{theorem}
	Let $Z$ be an arbitrary zonotope and $Z' =\mathrm{Proj}_{R_{\mathrm{n}}} (Z)$, then  $\mathrm{ReLU} (Z)=\mathrm{ReLU} (Z')$. 
	$\mathrm{Proj}_M \left(Z\right) = Z' = \left(c' \mid G'\right)$ is defined by 
	$c'_d = 0~\mathrm{if~} d \in M, \mathrm{~else~} c'_d = c_d$ and 
	$g'_{i,d} = 0~\mathrm{if~} d \in M, \mathrm{~else~} g'_{i,d} = g_{i,d}$.
	\label{theo:proj_whole_d}
\end{theorem}
\begin{proof}
	Applying ReLU and the projection operator to~$Z$ results in the sets: 
	\begin{equation*}
	\begin{aligned}
	\mathrm{ReLU} \left(Z\right)                    &= \left\{ a \mid a_d = \max\{0, b_d\}, b \in Z \right\} \\
	\mathrm{Proj}_{R_{\mathrm{n}}} \left(Z\right) &= \left\{ q \mid q_d = \left\{ \begin{array}{ll}
	0~\mathrm{if}~d \in R_{\mathrm{n}} \\
	p_d~\mathrm{else} \\
	\end{array} \right., p \in Z
	\right\} \\
	&= \left\{ q \mid q_d = \left\{ \begin{array}{ll}
	0~\mathrm{if}~\forall p \in Z: p_d < 0 \\
	p_d~\mathrm{else} \\
	\end{array} \right., p \in Z
	\right\}
	\end{aligned}
	\end{equation*}
	Applying ReLU on the projection results in $\mathrm{ReLU}(Z)$: 
	\begin{equation*}
	\begin{aligned}
	\mathrm{ReLU} \left( \mathrm{Proj}_{R_{\mathrm{n}}} \left(Z\right) \right) 
	&= \left\{ a \mid a_d = \max\{0, b_d\}, b \in \mathrm{Proj}_{R_{\mathrm{n}}} \left(Z\right) \right\} \\
	&= \left\{ a \mid a_d = \max\{0, q_d\}, \right.
	\left. q_d = \left\{ \begin{array}{ll}
	0~\mathrm{if}~\forall r \in Z: r_d \leq 0 \\
	p_d~\mathrm{else} \\
	\end{array}, \right.
	p \in Z
	\right\}  \\
	&= \left\{ a \mid a_d = \max\{0, p_d\}, p \in Z \right\}  
	= \mathrm{ReLU} \left(Z\right) 
	\end{aligned}
	\end{equation*} 
	\qed
\end{proof}
Projecting~$Z$ results in the more compact~$Z'$ with no change to the output set. We overload notation, and $Z$ denotes the projected zonotope in the following.

\textbf{Computation of Quadrants that Contain a Subset~S$_k$ of Z.}
Next, we subdivide the projected zonotope~$Z$ into subsets located in one quadrant. 
Quadrants that contain points of~$Z$ are determined by an index set $R_k$, where $R_k$ is an element of the power set $\mathcal{P} (R)$ of $R$. 
Each index set~$R_k$ corresponds to a set $S_k = \{p \mid p \in Z \wedge p_d \leq 0 \forall d \in R_k \wedge p_d \geq 0 \forall \notin R_k \}$. 
Clearly, all~$S_k$ are convex and disjoint, the union over~$S_k$ results in~$Z$. 
It is important to highlight that we \textit{never materialize the subsets $S_k$}, as they are unfavorable to compute. 
Our core idea is to approximate each~$S_k$ by zonotope~$\hat{Z}_k$. Subsequently, we project $\hat{Z}_k$ in all dimensions of the corresponding~$R_k$ (see case~1), resulting in $\mathrm{Proj}_{R_k} (\hat{Z}_k)$. The obtained set of zonotopes is an approximatiuon of~$\mathrm{ReLU} (Z)$ and is the input for the next layer. 
The computation of $R$, $R_k$ and corresponding subsets~$S_k$ is illustrated in the following example. 
\begin{example}
	Consider zonotope~$Z = (c \mid G)$ (Figure~\ref{fig:theory_approx}), where $c=(6,1)$ and generators are $g_1 = (3, 0)$, $g_2 = (2, 3)$ and $g_3 = (0, 0.5)$. 
	The lower bounds are $(1, -2.5)$, the upper bounds are $(11, 4.5)$. 
	As all upper bounds are positive, we do not project any dimension. The index set considering case~3 is $R = \{2\}$. 
	We need to approximate all subsets~$S_k$ corresponding to $R_k \in \mathcal{P} (R)$. 
	The empty set corresponds to the positive quadrant. 
\end{example}

We capture each~$S_k$ individually to keep the approximations as tight as possible.
Theoretically, we could decrease the number of zonotopes by over-approximat\-ing several~$S_k$ by one zonotope or by not considering small subsets~$S_k$ in the case of under-approximation. 
We discuss such an extension that restricts the maximum number of zonotopes at the end of this section. This extension enables a balance between tightness of approximations and run-time, which is useful for larger neural networks. 
The approximation of~$S_k$ can be either an over-approximation (RsO) or an under-approximation (RsU).

\textbf{Over-approximation of ReLU(Z).}
Given $S_k \subseteq Z$, we aim to over-approxi\-mate $S_k$ by $\hat{Z}_k = (\hat{c} \mid \hat{G})$ (to unclutter the notation, we omit the index~$k$ w.r.t.\ center and generators).
Our core idea is that if $\hat{Z}_k$ is a tight over-approximation of~$S_k$, the shape of $\hat{Z}_k$ should resemble the shape of~$Z$ (see  Figure~\ref{fig:over_z}). 
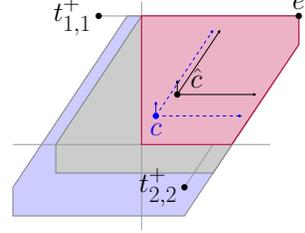
\begin{wrapfigure}{r}{0.29\textwidth}
	\centering
	\resizebox{0.29\textwidth}{!}{
		\begin{tikzpicture}[>={Stealth[inset=0pt,length=3.5pt,round]}, scale=0.9]

\coordinate (c) at (0,0);
\coordinate (g1) at (3.0,0.0);
\coordinate (g2) at (2.0,3.0);
\coordinate (g3) at (0.0,0.5);

\coordinate (g2p) at (1.0, 1.5);

\coordinate (x1) at(-5,-1);
\coordinate (x2) at(4,-1);
\coordinate (y1) at(-0.5,-4);
\coordinate (y2) at(-0.5,4);

\coordinate (e1) at (-5.0, -3.5);
\coordinate (e2) at (-5.0, 3.5);
\coordinate (e3) at (5.0, 3.5);
\coordinate (e4) at (5.0, -3.5);

\coordinate (f1) at (-5.0, -1.0);
\coordinate (f2) at (-5.0, 3.5);
\coordinate (f3) at (5.0, 3.5);
\coordinate (f4) at (5.0, -1.0);

\coordinate (t1) at (5.0, 3.5);
\coordinate (t2) at (1.0, -2.5);
\coordinate (t3) at (-2.0, 3.5);

\coordinate (s1) at (-0.5, 3.5);
\coordinate (s2) at (-0.5, -1.0);
\coordinate (s3) at (2.65, -1.0);

\filldraw [draw=black!50,fill=blue!20]
($(c) + (g3) + (g2) + (g1)$) --
($(c) + (g3) + (g2) - (g1)$) --
($(c) + (g3) - (g2) - (g1)$) --
($(c) - (g3) - (g2) - (g1)$) --
($(c) - (g3) - (g2) + (g1)$) --
($(c) - (g3) + (g2) + (g1)$) --
cycle;

\coordinate (cl) at (0.75, 0.75);
\coordinate (l1) at (2.75,0.0);
\coordinate (l2) at (1.5,2.25);
\coordinate (l3) at (0.0,0.5);
\filldraw[draw=black!50,fill=gray!40]
($(cl) + (l3) + (l2) + (l1)$) --
($(cl) + (l3) + (l2) - (l1)$) --
($(cl) + (l3) - (l2) - (l1)$) --
($(cl) - (l3) - (l2) - (l1)$) --
($(cl) - (l3) - (l2) + (l1)$) --
($(cl) - (l3) + (l2) + (l1)$) --
cycle;

\draw [thick, black!40] (x1) -- (x2) node[anchor=north] {$ $};
\draw [thick, black!40] (y1) -- (y2) node[anchor=north] {$ $};

\draw [->, thick, black!40] (t1) -- (t2) node[anchor=north] {\Huge \textcolor{teal}{$ $}};
\draw [->, thick, black!40] (t1) -- (t3) node[anchor=north] {\Huge \textcolor{teal}{$ $}};

\filldraw[draw=purple!100,fill=purple!30]
(s1) --
(s2) --
(s3) --
($(c) - (g3) + (g2) + (g1)$) --
($(c) + (g3) + (g2) + (g1)$) --
cycle;

\draw [->, thick, dashed, blue!100] (c) -- ($(c) + (g2)$) node[anchor=north] {\Huge $ $};
\draw [->, thick, dashed, blue!100] (c) -- ($(c) + (g1)$) node[anchor=north] {\Huge $ $};
\draw [->, thick, blue!100] (c) -- ($(c) + (g3)$) node[anchor=north] {\Huge $ $};
\draw [->, thick, black!100] (cl) -- ($(cl) + (l2)$) node[anchor=north] {\Huge $ $};
\draw [->, thick, black!100] (cl) -- ($(cl) + (l1)$) node[anchor=north] {\Huge $ $};
\draw [->, thick, black!100] (cl) -- ($(cl) + (l3)$) node[anchor=north] {\Huge $ $};

\node[inner sep=2pt, circle, fill=blue, label=below:\Huge \textcolor{blue}{$c$}] (cnode) at (c) {}; %
\node[inner sep=2pt, circle, fill=black, label={[xshift=0.6cm]\Huge \textcolor{black}{$\hat{c}$}}] (cnode) at (cl) {}; %

\node[inner sep=2pt, circle, fill=black, label=above:\Huge \textcolor{black}{$e$}] (cnode) at (t1) {}; %
\node[inner sep=2pt, circle, fill=black, label=left:\Huge \textcolor{black}{$t^{+}_{2,2}$}] (cnode) at (t2) {}; %
\node[inner sep=2pt, circle, fill=black, label=left:\Huge \textcolor{black}{$t^{+}_{1,1}$}] (cnode) at (t3) {}; %

\end{tikzpicture}
	}
	\caption{Overapproximation of \textcolor{magenta}{$S_0$} $\subseteq$ \textcolor{blue}{$Z$} by \textcolor{gray}{$\hat{Z}$}.} %
	\label{fig:over_extreme}
\end{wrapfigure}
As the shapes of two zonotopes are similar if their generators point in similar directions, we derive~$\hat{Z}_k$ from the generators of~$Z$. More precisely, the generators of $\hat{Z}_k$ are obtained by scaling each generator~$g_j$ of~$Z$ with a factor~$\alpha_j\in [0,1]$ and computing the shift of the center such that $S_k \subseteq \hat{Z}_k \subseteq Z$. Clearly, $\alpha_j = 1$ fulfills this property, but results in $\hat{g}_j = g$ and a loose over-approximation.
Thus, we aim to minimize over $\alpha_j$.
Each scaling factor~$\alpha_j$ for generator~$g_j$ is computed analytically (see Figure~\ref{fig:over_extreme}) 
by first computing an extreme point~$e$ of the zonotope. We start in~$e$ and test if the generator~$g_j$ allows a point~$t_{j,d}$ to be reached outside the quadrant under consideration. 
If this is the case, $g_j$ can be scaled down and~$\hat{Z}$ still over-approximates~$S_k$. 
We compute the extreme points and scaling factors for each dimension~$d$, resulting in $\alpha_{j,d}$. 

Regarding dimension~$d$, $g_j$ can be scaled by the factor $\alpha_{j,d}$. If we scale~$g_j$ by a larger factor, we leave the quadrant corresponding to~$S_k$ with respect to dimension~$d$. A larger scaling factor is not necessary in order to over-approximate~$S_k$. Thus, we minimize over~$\alpha_{j,d}$ to obtain the smallest~$\alpha_j$ and the tightest over-approximation. 
Formally, ~$\hat{g}_j$ and ~$\hat{c}$ of the over-approximating zonotope~$\hat{Z}_k$ are:
\begin{equation}
\begin{aligned}
	& \hat{g}_j = \alpha_j g_j, \mathrm{~~} \hat{c}   = c + \sum_j s_j (1-\alpha_j) o_j g_j \mathrm{~~with~the~following~definitions:}\\
	&\alpha_j                = \min_d \alpha_{j,d}, \mathrm{~~} d^* = \arg \min_d \alpha_{j,d}, \mathrm{~~} %
	 o_j                     = \frac{g_{j,d^*}}{\left| g_{j,d^*} \right|}, \mathrm{~~} s_j = 1 \mathrm{~if~} d^* \notin R_k, -1 \mathrm{~else} \\
	&\forall d \notin R_k:  t^{+}_{j,d} = c_d - 2 \left| g_{j,d} \right| + \sum_i \left| g_{i,d} \right|,	 
	 \alpha_{j,d} = 1 - \frac{|t^{+}_{j,d}|}{2 |g_{j,d}|}~~ \mathrm{if~}t^{+}_{j,d} < 0,~1 \mathrm{~else} \\
	&\forall d \in R_k:  t^{-}_{j,d} = c_d + 2 \left| g_{j,d} \right| - \sum_i \left| g_{i,d} \right|, 
	 \alpha_{j,d} = 1 - \frac{|t^{-}_{j,d}|}{2 |g_{j,d}|}~~ \mathrm{if~}t^{-}_{j,d} > 0,~1 \mathrm{~else} %
\label{eq:over_zono}
\end{aligned}
\end{equation}
Although the generators of~$Z$ are scaled down, the obtained zonotope $\hat{Z}$ is an over-approximation of~$S_k$ for the respective quadrant (which we never computed explicitly). This is shown in Theorem~\ref{theo:over_proof} by using Lemma~\ref{theo:over_const_1} and~\ref{theo:over_const_2}. 
\begin{theorem}
	Let~$Z = (c \mid G)$, $S_k = \{p \mid p \in Z \wedge p_d \geq 0 ~\forall d \notin R_k \wedge p_d \leq 0 ~\forall d \in R_k\}$ and~$\hat{Z}_k =(\hat{c} \mid \hat{G})$ with the center and generators as defined in Equation~\ref{eq:over_zono}. Then $S_k \subseteq \hat{Z}_k$. 
	\label{theo:over_proof}
\end{theorem}
\begin{proof}
	Let $p \in S_k$. Since  $p\in Z$ it exists $\beta_j \in [-1, 1]$ such that:
	\begin{equation}
	\begin{aligned}
	p &= c + \sum_j \beta_j g_j 
	= \hat{c} - \sum_j s_j (1-\alpha_j) o_j g_j + \sum_j \beta'_j o_j g_j \\
	&= \hat{c} + \sum_j (\beta'_j - s_j (1-\alpha_j)) o_j g_j 
	= \hat{c} + \sum_j \frac{\beta'_j - s_j (1-\alpha_j)}{\alpha_j} o_j \hat{g}_j \\
	\end{aligned}
	\end{equation}
	where we use that $o_{j} = \frac{g_{j,d^*}}{|g_{j,d^*}|} \in \{-1, 1\}$, $\beta'_j = o_j \beta_j$, $c = \hat{c} - \sum_j s_j (1-\alpha_j) o_j g_j$ and $g_j = \frac{1}{\alpha_i} \hat{g}_j$. 
	If we can show that $\frac{\beta'_j - s_j (1-\alpha_j)}{\alpha_j} o_j \in [-1, 1]$ then $p \in Z_k$. 
	To this end, we distinguish how $\alpha_j$ is obtained: If $\alpha_j$ is computed based on~$d^* \notin R_k$ then $s_j=1$ and it holds that $2(1-\alpha_j) -1 \leq \beta'_j$ (see Lemma~\ref{theo:over_const_1}). 
	If $\alpha_j$ is computed based on~$d^* \in R_k$ then $s_j=-1$ and $\beta'_j \leq 1 - 2 (1-\alpha_j)$ (see Lemma~\ref{theo:over_const_2}). 
	With these constraints and $\beta'_j \in [-1,1]$ (from the definition of zonotopes) we obtain $\frac{\beta'_j - s_j (1-\alpha_j)}{\alpha_j} \in [-1,1]$.
	Using~$o_j \in \{-1, 1\}$, we define $\hat{\beta}_j = \frac{\beta'_j - s_j (1-\alpha_j)}{\alpha_j} o_j$ and obtain $p = \hat{c} + \sum_j \hat{\beta}_j \hat{g}_j \in \hat{Z}_k$ $\Rightarrow \forall p \in S_k: p \in \hat{Z}_k \mathrm{~~and~~} S_k \subseteq \hat{Z}_k$. \qed
\end{proof}

\begin{lemma}
	Consider zonotope~$Z = (c \mid G)$. Let~$S_k = \{p \mid p \in Z \wedge p_d \geq 0 ~\forall d \notin R_k \wedge p_d \leq 0 ~\forall d \in R_k\}$ and let~$\hat{Z}_k$ be a zonotope with the center and generators defined in Equation~\ref{eq:over_zono}. Consider the definitions used in Theorem~\ref{theo:over_proof}. Then $2 (1-\alpha_j) - 1 \leq \beta'_j$ if $\alpha_j$ corresponds to a $d \notin R_k$.
	\label{theo:over_const_1}
\end{lemma}

\begin{proof}
	We use that for a point $p \in Z: p_d \geq 0$ in case $d \notin R_k$ and $t^{+}_{j,d} < 0$. 
	\begin{align}
	0                    &\leq p_d \\
	\Leftrightarrow     t^{+}_{j,d}  + 2 \frac{1}{2} \frac{|t^{+}_{j,d} |}{|g_{j,d}|} |g_{j,d}| 
	&\leq c_d + \sum_i \beta_i g_{i,d} \label{eq:over_c1_01}\\
	\Leftrightarrow     t^{+}_{j,d}  + 2 (1-\alpha_j) |g_{j,d}| 
	&\leq c_d + \sum_i \beta_i g_{i,d} \label{eq:over_c1_02}\\ 
	\Leftrightarrow     c_d + \sum_i \left|g_{i,d}\right| -2 \left|g_{j,d} \right| + 2 (1-\alpha_j) |g_{j,d}|
	&\leq c_d + \sum_i \beta_i g_{i,d} \label{eq:over_c1_03}\\
	\Leftrightarrow     \sum_{i, i \neq j} o_{i,d} g_{i,d} - |g^d_j| + 2 (1-\alpha_j) |g_{j,d}|
	&\leq \sum_{i, i \neq j} \beta_i g_{i,d} + \beta_j g_{j,d} \label{eq:over_c1_04}\\
	\Leftrightarrow     \sum_{i, i \neq j} (o_i^d - \beta_i) g_{i,d} + (2 (1-\alpha_j) - 1) |g_{j,d}|
	&\leq \beta'_j |g_{j,d}| \label{eq:over_c1_05}\\
	\Leftrightarrow     \underbrace{\frac{1}{|g_{j,d}|} \sum_{i, i \neq j} (o_{i,d} - \beta_i) g_{i,d}}_{\geq 0} + (2 (1-\alpha_j) - 1)
	&\leq \beta'_j \label{eq:over_c1_06}\\
	\Rightarrow          2 (1-\alpha_j) - 1   &\leq \beta'_j 
	\end{align}
	We use that $p_d = c_d + \sum_i \beta_i g_{i,d}$ (\ref{eq:over_c1_01}), from the definitions of $t^{+}_{j,d}$ 
	$1 - \alpha_{j,d} = \frac{|t^{+}_{j,d}|}{2 |g_{j,d}|} $ if $t^{+}_{j,d} < 0$ and $0$ else (\ref{eq:over_c1_02}), the definition of $t^{+}_{j,d}$ (\ref{eq:over_c1_03}), 
	$o_{j,d} g_{j,d} = \left|g_{j,d}\right|$ (\ref{eq:over_c1_04}) and $\beta_j g_{j,d} = \beta'_j |g_{j,d}|$ (\ref{eq:over_c1_05}). Inequality $\frac{1}{|g_{j,d}|} \sum_{i, i \neq j} (o_{i,d} - \beta_i) g_{i,d} \geq 0$ (\ref{eq:over_c1_06}) holds because
	$o_{i,d} g_{i,d} = |g_{i,d} |$ and 
	$\beta_i g_{i,d} = \pm \beta_i |g_{i,d}| \leq |g_{i,d} |$ because $\beta_i \in [-1,1]$
	$\Rightarrow~  (o_{i,d} - \beta_i) g_{i,d} \geq 0$ and thus,
	$\sum_{i, i \neq j} (o_{i,d} - \beta_i) g_{i,d} \geq 0$. \qed
\end{proof}

\begin{lemma}
	Consider zonotope~$Z = (c \mid G)$. Let~$S_k = \{p \mid p \in Z \wedge p_d \geq 0 ~\forall d \notin R_k \wedge p_d \leq 0 ~\forall d \in R_k\}$ and let~$\hat{Z}_k$ be a zonotope with the center and generators defined in Equation~\ref{eq:over_zono}. Consider the definitions used in Theorem~\ref{theo:over_proof}. Then $\beta'_j \leq 1 - 2 (1-\alpha_j)$ if $\alpha_j$ corresponds to a $d \in R_k$. 
	\label{theo:over_const_2}
\end{lemma}
The proof of Lemma~\ref{theo:over_const_2} is similar to the one of Lemma~\ref{theo:over_const_1} and not shown in detail. The differences to the previous proof are that for a point $p \in Z: p_d \leq 0$ in case $d \in R_k$ and $t^{-}_{j,d} > 0$, we start with $0 \geq p_d $, we use the  $t^{-}_{j,d}$ instead of $ t^{+}_{j,d}$ and signs of the terms are different.

The subset~$S_k$ is located in one quadrant and the corresponding~$R_k$ contains dimensions that are mapped to zero by ReLU (case~1 on $S_k$). 
Thus, we project the over-approximation~$\hat{Z}_k$ in dimensions $d \in R_k$ as described above. This projection is exact:
$\mathrm{ReLU} (S_k) = \mathrm{Proj}_{R_k} (S_k) \subseteq \mathrm{Proj}_{R_k} (\hat{Z}_k)$.

\textbf{Under-approximation of ReLU(Z).}
Finding a tight under-approximation of $S_k$ turns out to be more challenging. We propose to tackle this by solving a constrained optimization problem, 
in which we aim to find a zonotope $\hat{Z}_k$ of maximum volume 
subject to the constraint $\hat{Z}_k \subseteq S_k$:
\begin{equation}
\begin{aligned}
	\hat{Z}_k = \arg \max_{\hat{Z}} V(\hat{Z}) ~~\mathrm{subject~to~~} & \hat{Z} \subseteq S_k
\end{aligned}
\label{eq:exact_under}
\end{equation}
How can we instantiate Equation~\ref{eq:exact_under} to under-approximate~$S_k$ tightly and keep computations efficient? 
We derive an efficient linear program by considering the same search domain as before. 
More precisely, we constrain the search space to zonotopes that are derived from the original zonotope~$Z$, by scaling its generators~$g_i$ by factors~$\alpha_i \in \left[0, 1\right]$, i.e. $\hat{g}_i = \alpha_i g_i $. 
As motivated before, this assumption is reasonable, since~$\hat{Z}_k$ and~$Z$ have similar shapes.

Importantly, to ensure that we under-approximate a part of~$Z$ located in one quadrant, we add a constraint that forces the lower bound of the interval hull of~$\hat{Z}$ to be non-negative if $d \notin R_k$: $\hat{c}_d - \sum_i \left| \hat{g}_{i,d} \right| \geq 0 ~~ \forall d \notin R_k$ and one that forces the upper bound of the interval hull of~$\hat{Z}$ to be negative if $d \in R_k$: $\hat{c}_d + \sum_i \left| \hat{g}_{i,d} \right| \leq 0 ~~ \forall d \in R_k$. 
Since the volume of the zonotope grows with~$\alpha_i$, we instantiate the objective function by $\sum_i \alpha_i$.
Combining all of these considerations results in the following linear optimization problem: 
\begin{equation}
\begin{aligned}
	\alpha^*,\delta^* = \arg \max_{\alpha,\delta} \sum_{i} \alpha_{i} \mathrm{~~}  &\mathrm{subject~to~~} 
	\hat{g}_i = \alpha_i g_i , ~~~
	\alpha_i \in \left[0, 1\right] , ~~~ \\             
	&\hat{c} = c + \sum_i \delta_i g_i , ~~~
	\left|\delta_i\right| \leq 1 - \alpha_i \\
	 & \hat{c}_d - \sum_i \left| \hat{g}_{i,d} \right| \geq 0 ~\forall d \notin R_k , ~~~
	\hat{c}_d + \sum_i \left| \hat{g}_{i,d} \right| \leq 0  ~\forall d \in R_k \nonumber
\end{aligned}
\label{eq:under_opt}
\end{equation}

\begin{theorem}
	Let~$\hat{Z}_k$ be computed from zonotope~$Z$ based on $\alpha^*,\delta^*$,  %
	 then~$\hat{Z}_k \subseteq S_k$. 
	\label{theo:under}
\end{theorem}
\begin{proof}
	Let $\gamma_i = \frac{\delta_i}{1-\alpha_i}$. Since $\left|\delta_{i}\right| \leq 1-\alpha_i$ it holds that $\gamma_i \in \left[-1, 1\right]$. Since~$p \in \hat{Z}_k$ $\hat{\beta}_i \in \left[-1,1\right]$ exists:
	\begin{align*}
	p &= \hat{c} + \sum_i \hat{\beta_i} \hat{g_i} 
	= c + \sum_i \delta_i g_i + \sum_i \hat{\beta_i} \alpha_i g_i \\
	&= c + \sum_i \gamma_i \left(1-\alpha_i\right) g_i + \sum_i \hat{\beta_i} \alpha_i g_i 
	= c + \sum_i \left( \left( 1 - \alpha_i \right) \gamma_i + \alpha_i \hat{\beta}_i \right) g_i 
	\end{align*}
	To prove that $p \in Z$ we need to show that $\left( 1 - \alpha_i \right) \gamma_i + \alpha_i \hat{\beta}_i \in [-1, 1]$. 
	Considering $\gamma_i \in [-1, 1]$, $\alpha_i \in [0,1]$ and $\hat{\beta_i} \in [-1,1]$ we obtain: 
	\begin{align*}
	\forall i: \left( 1 - \alpha_i \right) \gamma_i + \alpha_i \hat{\beta}_i 
	&\geq - \left( 1 - \alpha_i \right)  - \alpha_i 
	= - 1 \\
	\forall i: \left( 1 - \alpha_i \right) \gamma_i + \alpha_i \hat{\beta}_i 
	&\leq \left( 1 - \alpha_i \right) + \alpha_i
	= 1 
	\end{align*}
	Thus, we define $\beta_i = \left( 1 - \alpha_i \right) \gamma_i + \alpha_i \hat{\beta}_i$ and obtain $p = c + \sum_i \beta_i g_i$ and thus,~$\hat{Z}_k \subseteq Z$.
	The constraints~$\hat{c}_d - \sum_i \left| \hat{g}_{i,d}  \right| \geq 0 \forall d \notin R_k$ and $\hat{c}_d + \sum_i \left| \hat{g}_{i,d} \right| \leq 0 \forall d \in R_k$ ensure that $\hat{Z}$ is located in the desired quadrant: 
	\begin{align*}
	p_d &= \hat{c}_d + \sum_i \hat{\beta_i} \hat{g_{i,d}} 
	\geq \hat{c}_d - \sum_i \left| \hat{g}_{i,d} \right| 
	\geq 0 \mathrm{~~if~~} d \notin R_k \\
	p_d &= \hat{c}_d + \sum_i \hat{\beta_i} \hat{g_{i,d}} 
	\leq \hat{c}_d + \sum_i \left| \hat{g}_{i,d} \right| 
	\leq 0 \mathrm{~~if~~} d \in R_k
	\end{align*}
	Thus, $\hat{Z}_k \subseteq S_k$. \qed
\end{proof}

If the quadrant under consideration is empty (which can happen for many quadrants) the optimization problem is not solvable and we can safely ignore this quadrant. 
Since all points in $\hat{Z}_k$ are negative w.r.t. the dimensions~$d \in R_k$ (case~1), we compute $\mathrm{Proj}_{R_k} (\hat{Z}_k)$ and obtain an under-approximation of $\mathrm{ReLU} (S_k)$.  See Figure~\ref{fig:under_z} for illustration.

\hypertarget{balancing}{\textbf{Balancing approximation tightness and run-time.}}\label{sec:balancing}
For large input sets, the number of convex subsets that define the reachable set of a neural network scales exponentially with the number of neurons. 
Let us consider zonotope~$Z = (c \mid G), ~ c \in \mathbb{R}^D, ~ G \in \mathbb{R}^{n \times D}$. In the worst case, $Z$ consists of points that are spread over~$2^D$ quadrants. RsO and RsU approximate each subset~$S_k$ located in one quadrant by a separate zonotope~$Z_k$. 

To balance approximation tightness and run-time, we extend RsO and RsU, such that the number of zonotopes can be restricted by the user. The overall number of zonotopes (w.r.t.\ the whole neural network) is restricted by~$B$ and the amplification is restricted by~$A$. The amplification is the maximum number of zonotopes $Z_k$ used to approximate~$\mathrm{ReLU}(Z)$ w.r.t.\ one layer and one input zonotope~$Z$. It is defined by the number of quadrants~$q$ that contain points of~$Z$ and can be computed as follows. First, we compute the interval hull of~$Z$. 
With respect to dimension~$d$, all points within~$Z$ are in the interval $\left[ l_{\mathrm{low}}^d, l_{\mathrm{upp}}^d \right] = \left[c_d - \delta g_d, c_d + \delta g_d \right]$, where~$\delta g = \sum_i |g_i|$. Second, we count the number~$R_n$ of dimensions~$d$ where $l_{\mathrm{low}}^d < 0$ and $l_{\mathrm{upp}}^d > 0$.  The number of quadrants is $q = 2^{R_n}$. 

Over-approximation: If $q > A$, we compute the interval hull of~$Z$ and restrict the intervals to their positive portion. Thus, we over-approximate $\mathrm{ReLU} (Z)$ by one zonotope instead of~$q$ zonotopes. 
If the overall number of zonotopes is larger than~$B$, we estimate the size of each zonotope~$Z = (c \mid G)$ by 
$\mathrm{size}\left(Z\right) = \sum_d \log (\delta g)_d$, where $\delta g = \sum_{i} |g_i|$.
The largest~$B-1$ zonotopes are kept while the smaller ones are merged (i.e. we compute an over-approximation of their union by minimizing/maximizing over the lower/upper limits of their interval hulls). The resulting interval is transformed into the G-representation of a zonotope:
$l_{\mathrm{low}} = \min_{Z_k} (c_k - (\delta g)_k)$, 
$l_{\mathrm{upp}} = \max_{Z_k} (c_k + (\delta g)_k)$,
$Z_{\mathrm{uni}} = ( l_{\mathrm{low}} + g_{\mathrm{ext}} \mid \mathrm{diag} (g_{\mathrm{ext}}))$,
$g_{\mathrm{ext}} = 0.5 (l_{\mathrm{upp}} - l_{\mathrm{low}} )$
where $\mathrm{diag} (g_{\mathrm{ext}})$ is a diagonal matrix. 

Under-approximation: In this case, we simply drop the smallest zonotopes if $q > A$ or the overall number of zonotopes is larger than~$B$.

\section{Applications and Experiments}
\label{sec:app_res}

We highlight the versatility of our RsO and RsU approach by describing several applications in classification and regression tasks. 
More specifically, we discuss (non-)robustness verification, robust training, quantification of feature importance and the distinction between reliable and non-reliable predictions. 
Furthermore, we analyze reachable sets of an autoencoder. 
Our input zonotopes capture three different shapes: cube (equivalent to $L_{\infty}$-norm), box (with a different perturbation on each feature) and free (with coupling of features). 
We train feed-forward ReLU networks on standard data sets.

\textbf{Experimental Setup.}
Our approaches are implemented in Python/Pytorch. We train feed-forward ReLU networks using stochastic gradient descent with cross-entropy loss (classifiers), Huber loss (regression models), mean-square-error loss (autoencoder models) or robust loss functions (see following sections) and early stopping. Experiments are carried out on the following popular data sets and neural network architectures (accuracy denotes worst accuracy obtained for this data set by one of the specified neural network architectures): 
\textbf{\underline{Classifiers}}: 
\textbf{Iris} \cite{iris, ucirepo}: $3$ classes, $4$ features, $1-5$ hidden layers of $4$ neurons each. %
\textbf{Wine} \cite{wine, ucirepo}: $3$ classes (cultivars), $13$ features, $1-5$ hidden layers of $6$ neurons each. %
\textbf{Tissue} \cite{tissue, ucirepo}: breast tissue probes, $6$ classes, $9$ features, $1-3$ hidden layers of~$8$ neurons each. %
Breast \textbf{cancer} Wisconsin (diagnostic) \cite{cancer, ucirepo}: $2$ classes, $30$ features, $1-2$ hidden layers of $10$ neurons each. %
\textbf{MNIST} \cite{mnist}: $28 \times 28$ gray-scale images, $10$ classes, $1$ hidden layer of~$15$ neurons. %
\textbf{Fashion-MNIST} \cite{fashionmnist}: $28 \times 28$ gray-scale images, $10$ classes, $1$ hidden layer of~$15$ neurons or $5$ hidden layers of~$30$ neurons. %
\textbf{\underline{Regression Models}}: 
\textbf{Abalone} \cite{abalone}: $8$ features, $1$ output, $1-3$ hidden layers of~$6$ neurons each. 
\textbf{Housing} \cite{housing}: $13$ features, $1$ output, $1$ hidden layer of~$13$ neurons.
\textbf{Airfoil} \cite{airfoil}: $5$ features, $1$ output, $1-4$ hidden layers of~$5$ neurons each.
\textbf{\underline{Autoencoder}}: 
\textbf{MNIST} \cite{mnist}: $28 \times 28$ gray scale images, $28 \times 28$ output, $3$ hidden layers of~$30 \times 60 \times 30$ neurons. 
\textbf{Fashion-MNIST} \cite{fashionmnist}: $28 \times 28$ gray scale images, $28 \times 28$ output, $3$ hidden layers of~$30 \times 60 \times 30$ neurons. 

For classification, all data sets are balanced by sub-sampling training- and test-sets such that evaluation experiments are done on the same amount of points for each class. The input size of MNIST and Fashion-MNIST is reduced from $28 \times 28$ to~$30$ by using principle component analysis (PCA). 
In the evaluation experiments, we use $30$ input points of the iris data set, wine data set and tissue data set, $86$ points of the cancer data set, $200$ point of the MNIST data set and $100$ points of the Fashion-MNIST data set, which are not part of the training set.

Experiments are conducted in Python (version 3.6) on a machine with 10 Intel Xeon CPU cores with 2.2 GHz, 4 GEFORCE GTX 1080 Ti and 256 GB of RAM running Ubuntu (version 16.04.6).

\textbf{Definition of Input Sets.}
Using zonotopes as input sets has the advantage that we are able to verify different kinds of perturbations. Here, the input set~$\hat{Z} = (\hat{c} \mid \hat{G})$ is defined by using an input data point~$x$ as center~$\hat{c}$ and the following perturbations specified by the generator matrix~$\hat{G}$.  
\textbf{Cube}: $\hat{Z}_{\mathrm{cube}}$ is a hyper-cube whose shape is equivalent to the unit ball of the~$L_{\infty}$-norm. As the allowed perturbation on each input feature is the same, the generator matrix is $\varepsilon I_d$ for different~$\varepsilon$. 
\textbf{Box}: $\hat{Z}_{\mathrm{box}}$ is a so called axis-aligned parallelotope ($n$-dimensional box). This shape allows different disturbances on each input feature, but it does not couple features.  For this, we first compute a zonotope by using the eigen-vectors that correspond to the~$d$ largest eigenvalues of the data set as generators. 
$Z_{\mathrm{box}}$ is obtained by computing the interval hull of this zonotope and scaling its volume such that it is equivalent to the volume of $\hat{Z}_{\mathrm{cube}}$ for a given~$\varepsilon$.
\textbf{Free}: $\hat{Z}_{\mathrm{free}}$ is an arbitrary zonotope that enables disturbances to be coupled between input features which cannot be captured by norms or intervals. This input zonotope is obtained by increasing/decreasing all feature values simultaneously by at most~$\varepsilon$ and additionally allowing a small, fixed perturbation~$\delta \ll \varepsilon$ on each feature. If the input is an image, this perturbation would brighten/darken all pixel values simultaneously: $\hat{G} = [\delta I_d, \varepsilon \vec{1}]$.

For feature rankings, the following setting is used: 
to quantify the influence of feature~$f_1$ on the prediction of~$x$, we define a box-shaped input set~$\hat{Z}_{f_1} = (x | G)$ around~$x$ that allows a perturbation~$\delta$ on~$f_1$ and a minimal perturbation $\varepsilon$  (here: $\varepsilon=0.01$) on all other features. More formally, we use a diagonal input matrix~$G$, where~$G_{1,1} = \delta$ and~$G_{i,i} = \varepsilon~ \forall i \neq 1$.

\textbf{Classification: (Non-)Robustness Verification.}
First, we evaluate the potential of reachable sets by using them for robustness/non-robustness verification, i.e. for studying how predictions of a classifier change when perturbing input instances. 
More precisely, we aim to analyze if predictions based on an input set map to the same class or if they vary.
Formally, the set of predictions (classes) is $P = \{\arg \max_c f(x)_c | x \in I\}$, given input set $I$.

For verification, we compute a robustness score against each class. 
Let~$a$ be the predicted class and $b \neq a$ the class against which we quantify robustness.\footnote{Please note that reachable sets capture all classes jointly. More precisely, the method does not require any label/class information at all. Thus, it is directly applicable to other tasks such as regression.} 
The least robust point~$p$ within the reachable set (output/logit space) is the one where its coordinate $p_b$ is close to or larger than~$p_a$. Based on these considerations, we define the robustness score against class~$b$ of reachable set $R_S$: 
\begin{equation}
\begin{aligned}
\mathrm{s}_b &= \min_{p \in R_S} \left(p_a - p_b \right) %
= \min_{Z = (c \mid G) \in R_S} \left( c_a - c_b - \sum_i \left| g_i^a - g_i^b \right| \right) 
\end{aligned}
\label{eq:scores}
\end{equation}
where $Z\in R_S$ denotes the computed zonotopes, and we use that $p_a = c_a + \sum_i \beta_i g_i^a$, $p_b = c_b + \sum_i \beta_i g_i^b$ and $\sum_i \beta_i (g_i^a - g_i^b)$ is minimal if~$\beta_i \in \{-1, 1\}$ depending on the sign of $g_i^a - g_i^b$. 

Robustness certificates are obtained by computing the scores against {all} classes $b\neq a$ on the \textit{over}-approximated reachable set $R_{SO}$. If \textit{all} scores are \textit{positive}, the robustness certificate holds, and all points from the input set are classified as class~$a$. Non-robustness certificates are obtained by checking if there is a class~$b$, such that $s_b$ on the \textit{under}-approximated reachable set $R_{SU}$ is \textit{negative}. If this is the case, at least one point from the input set is categorized as class~$b$.

There are three benefits to these scores.  
First, computing scores is efficient (see Equation~\ref{eq:scores}). What is more, the scores are fully differentiable w.r.t.\ the network parameters, enabling immediate robust training (see later experiment). 
Second, the scores are applicable to \textit{class-specific verification} (i.e. robust against class $b_1$, non-robust against~$b_2$). 
And thirdly, the scores allow relative quantification of (non-)robustness. A reachable set with a high score is more robust than one with a low score.

We compare the performance of \textsc{RsO} on robustness verification using the state-of-the-art methods, \textsc{wk} (wong-kolter)~\cite{wong2017},  dz (deepzono)~\cite{deepzono}, dp (deeppoly)~\cite{deeppoly}, dr (refinezono)~\cite{refinezono} and es (exact approach)~\cite{exactreach}, which computes the exact reachable set (implementation~\cite{julia_toolbox}).
RsU is compared with the success rate of FGSM attacks~\cite{goodfellow2014, szegedy2014} and PGD attacks~\cite{pgd_attack}. 
To handle the box setting, FGSM attacks are scaled, such that the perturbed input is contained within the input zonotope. The PGD attack is projected onto the input zonotope in each step, i.e.\  extended to handle arbitrary input zonotopes. 
Figure~\ref{fig:res_comp} and~\ref{fig:res_additional} illustrate (non-)robustness verification on the cancer data set, MNIST, iris data set and FashionMNIST for cube-, box- and free-shaped input zonotopes. 
\begin{figure}[ht]
	\centering
	\includegraphics[width=0.32\textwidth]{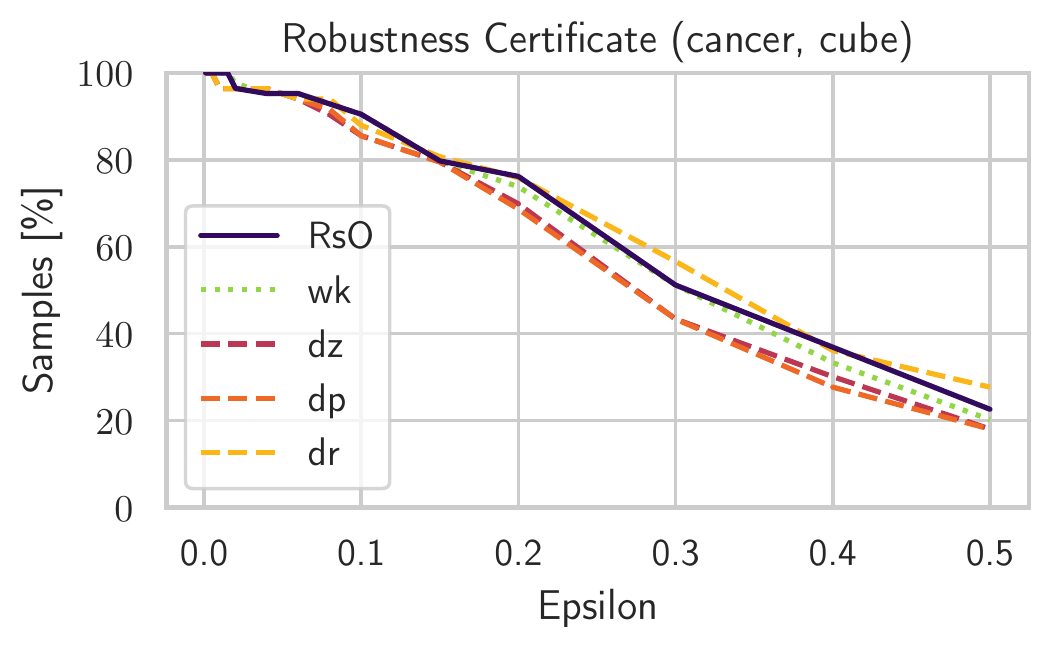}
	\includegraphics[width=0.32\textwidth]{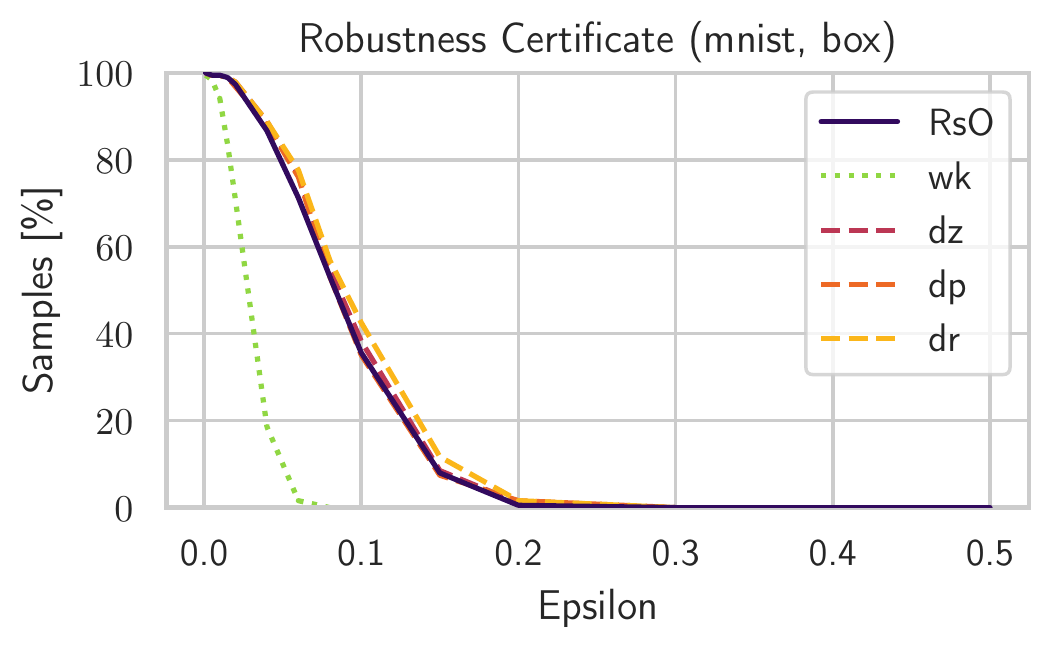}
	\includegraphics[width=0.32\textwidth]{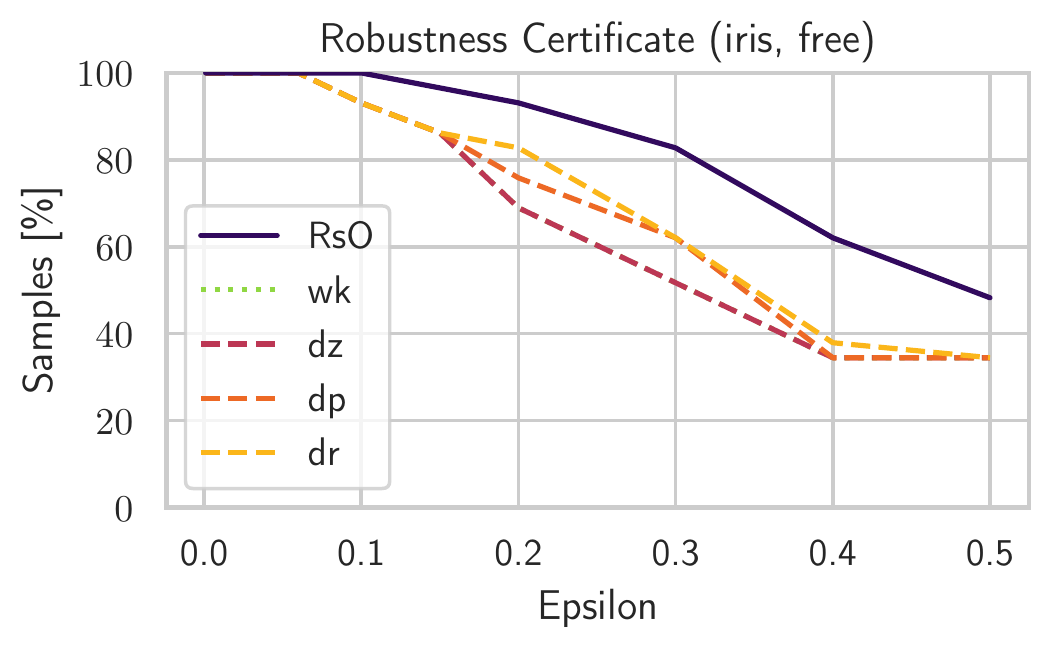}
	\includegraphics[width=0.32\textwidth]{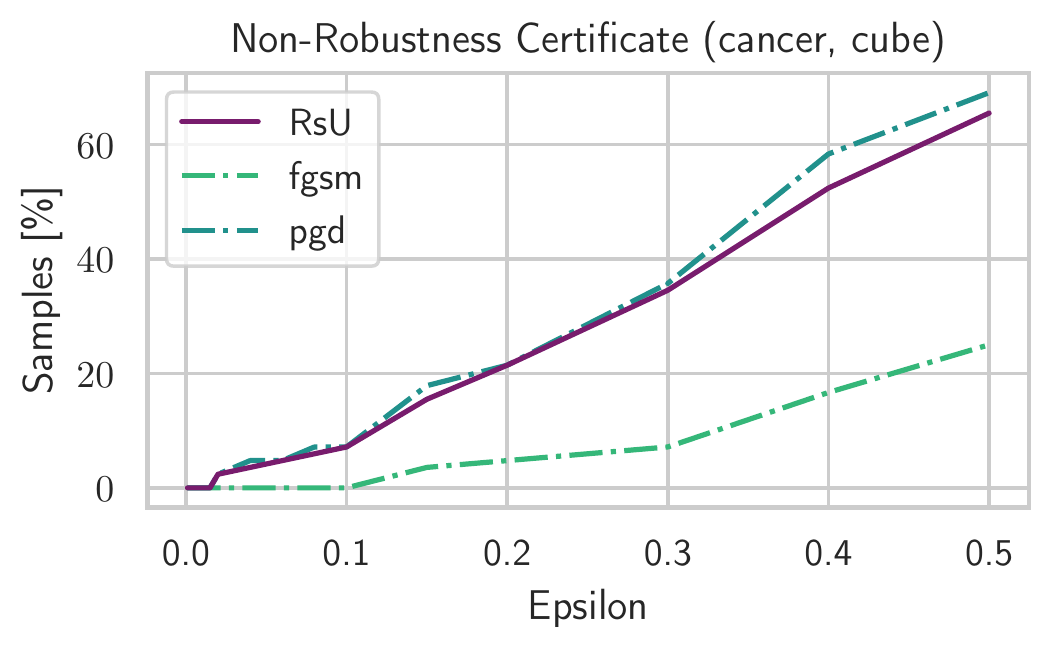}
	\includegraphics[width=0.32\textwidth]{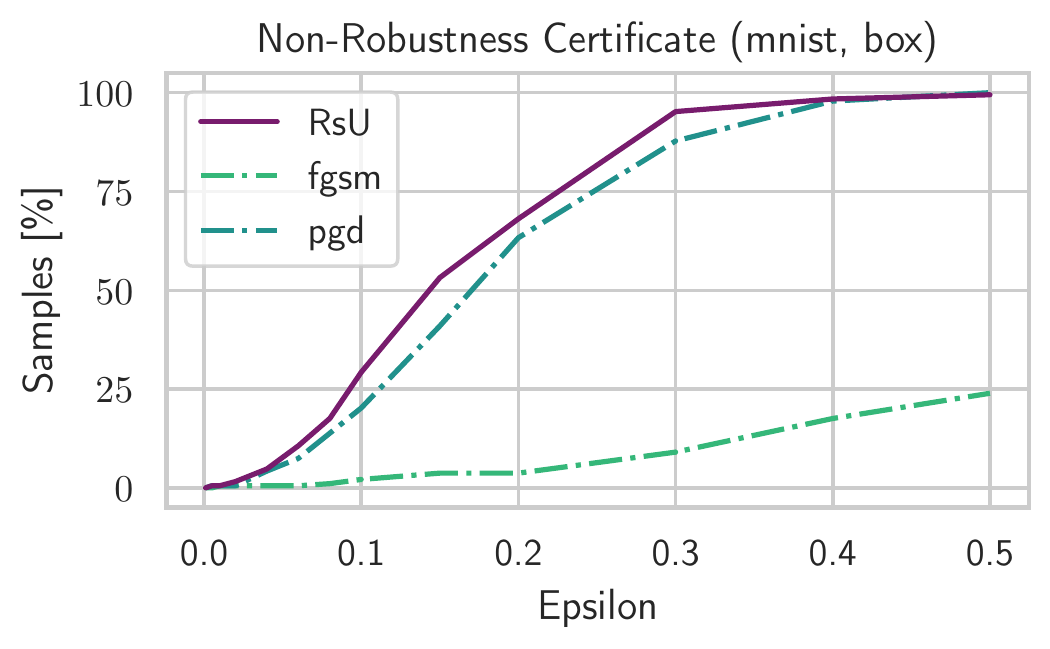}
	\includegraphics[width=0.32\textwidth]{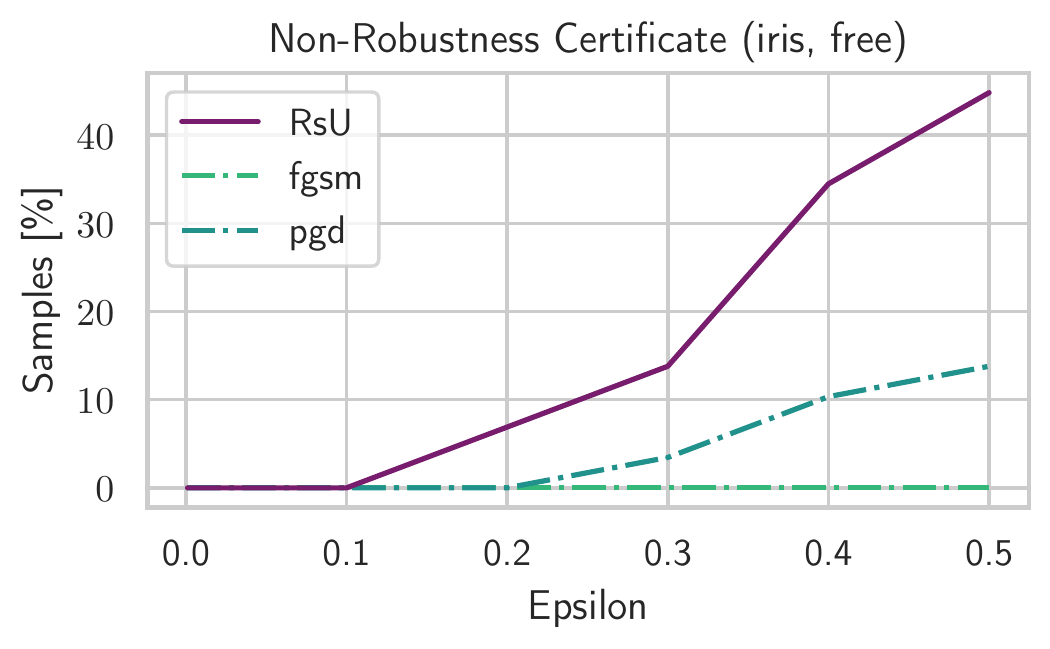}
	\caption{Performance evaluation of RsO and RsU in (non-)robustness verification on the cancer data-set (left, 2 hidden layers, acc. $97\%$), MNIST (middle, 1 hidden layer, acc. $94\%$) and iris (right, 5 hidden layers, acc. $97\%$).}
	\label{fig:res_comp}
\end{figure}
For robustness verification, we measure the number of samples for which the scores against all non-target classes are positive. 
For non-robustness verification, we count the number of samples in which a negative score exists against a class. 
In the cube and box settings, RsO perform similar way to dr, dz and dp, while RsU is similar (cube) or slightly better (box) than PGD attacks. Based on arbitrary input zonotopes (free setting), RsO and RsU outperform both state-of-the-art robustness verification approaches and PGD attacks. 
\begin{figure}[ht!]
	\centering
	\includegraphics[width=0.32\textwidth]{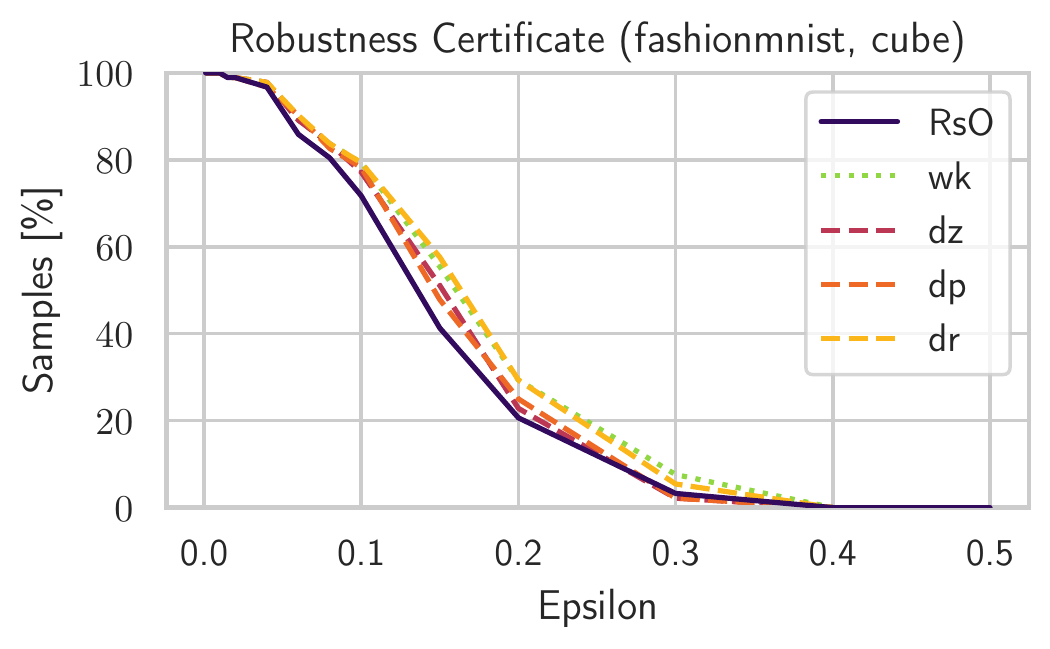}
	\includegraphics[width=0.32\textwidth]{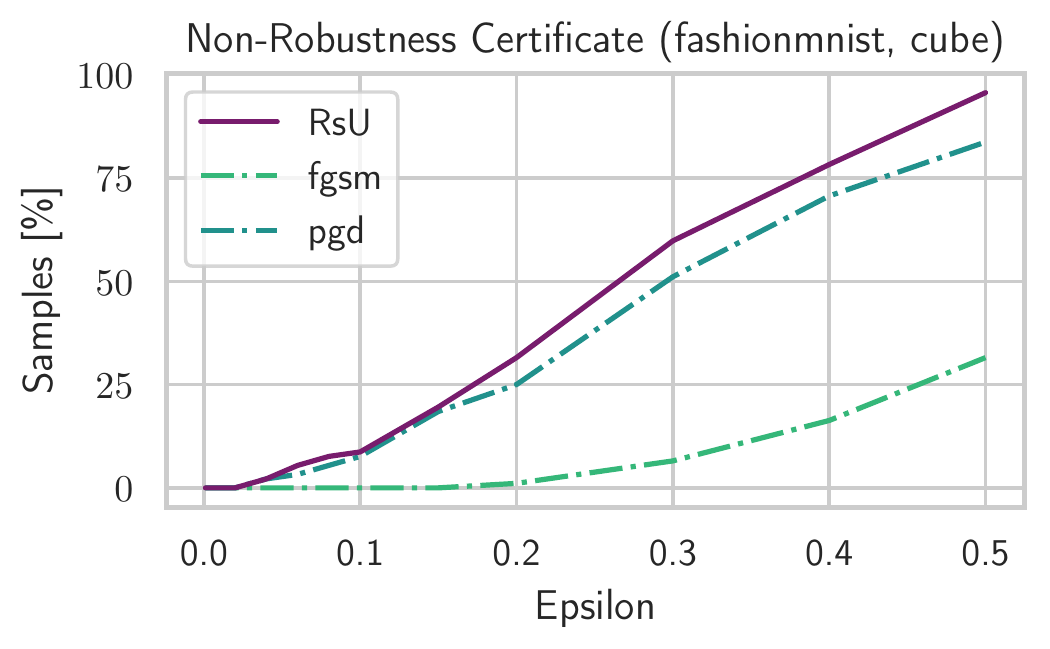}
	\includegraphics[width=0.32\textwidth]{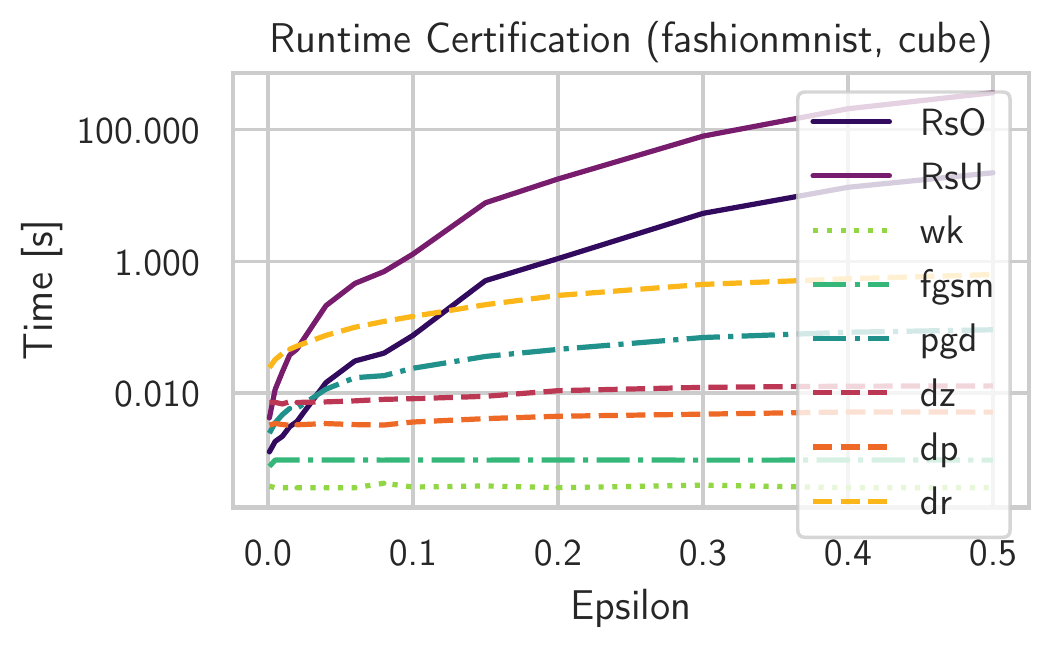}
	\includegraphics[width=0.32\textwidth]{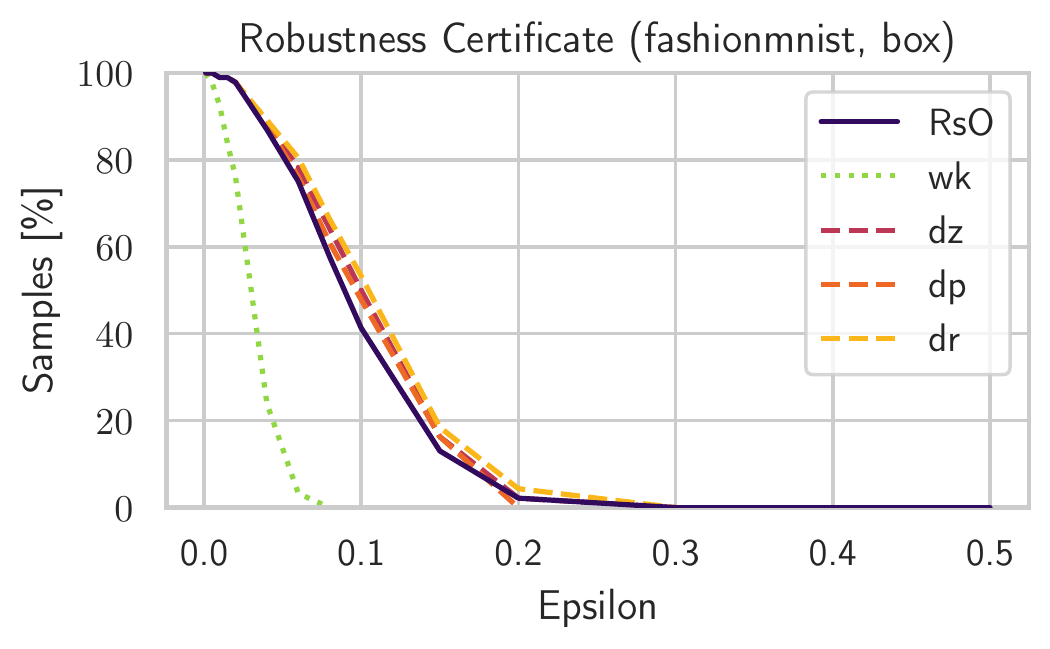}
	\includegraphics[width=0.32\textwidth]{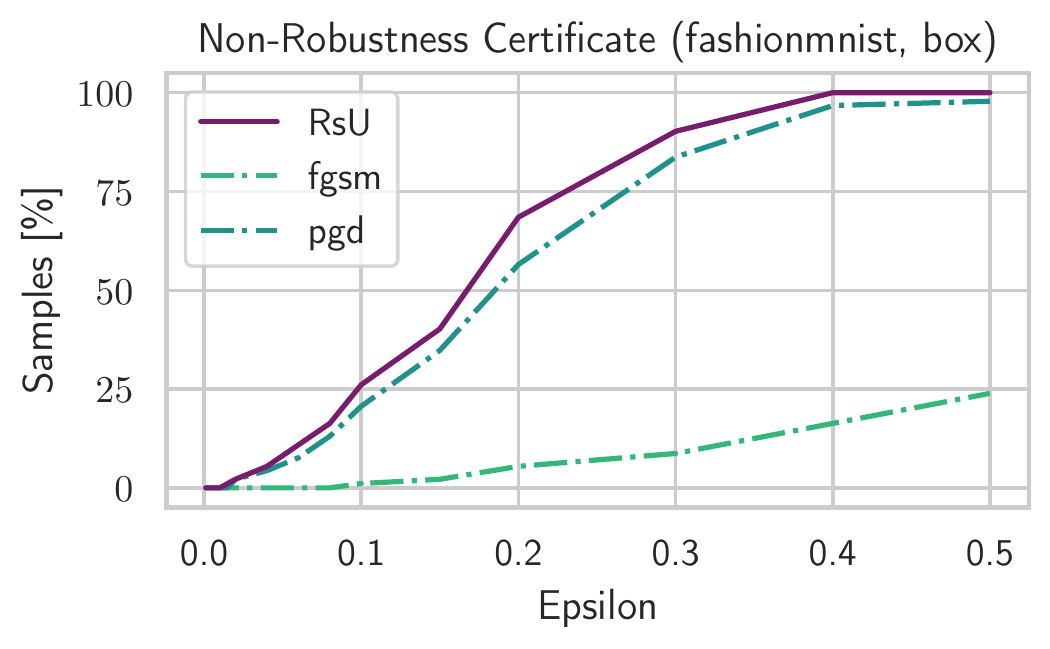}
	\includegraphics[width=0.32\textwidth]{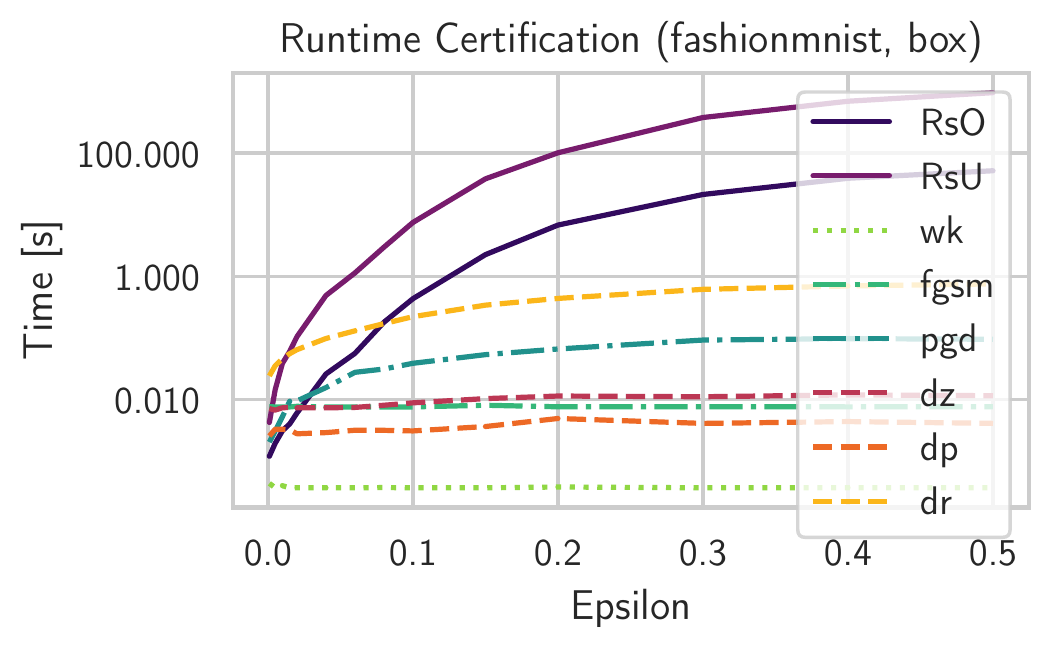}
	\includegraphics[width=0.32\textwidth]{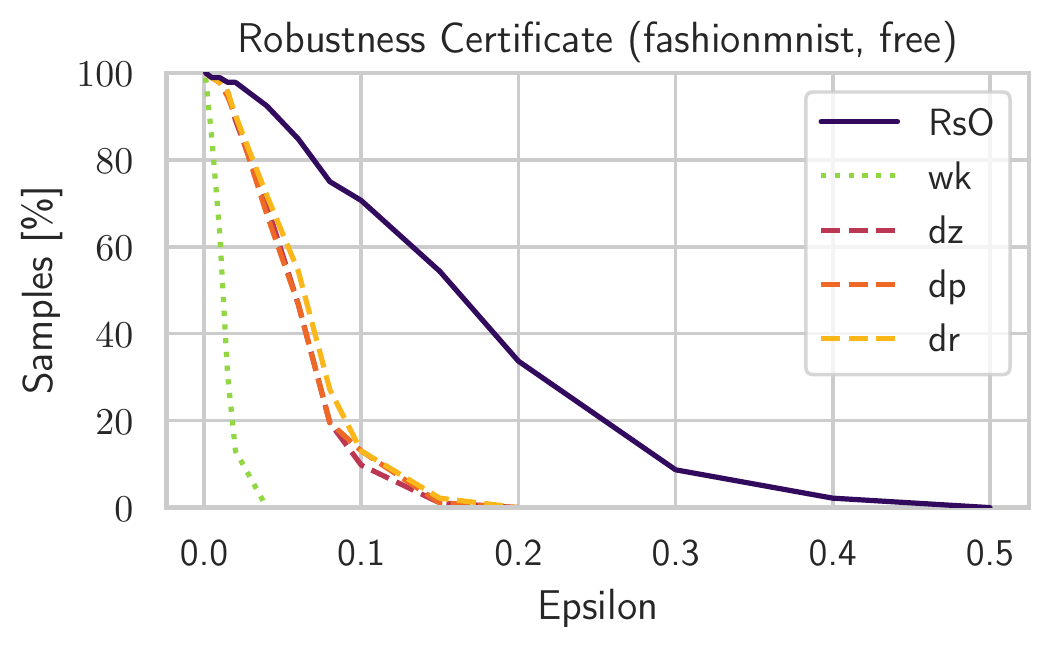}
	\includegraphics[width=0.32\textwidth]{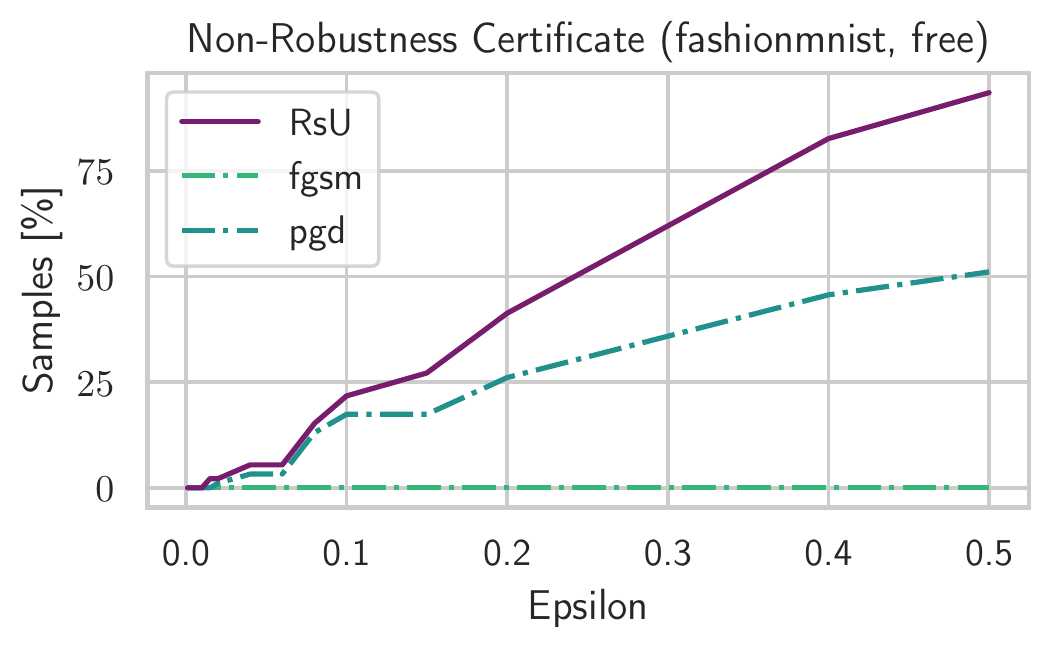}
	\includegraphics[width=0.32\textwidth]{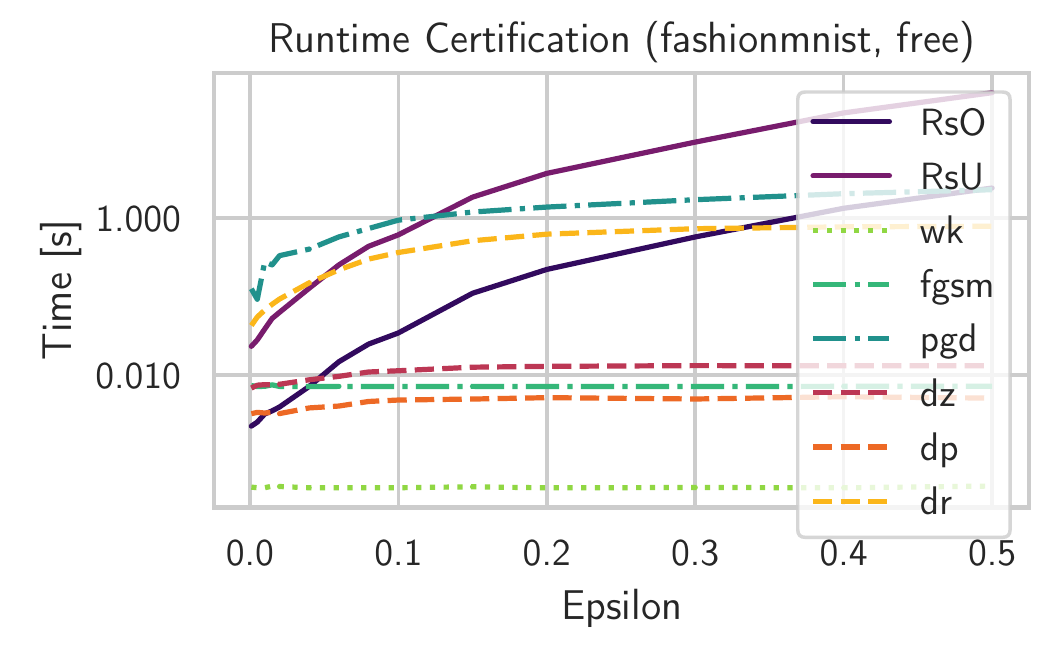}
	\caption{Performance evaluation and run-time of RsO and RsU in (non-)robustness verification on Fashion-MNIST (acc. $92\%$) using cube-shaped (top), box-shaped (middle) and freely-shaped input sets (bottom).}
	\label{fig:res_additional}
\end{figure}
The run-time of RsO and RsU increases with the number of input features, the number of neurons in the neural network and the perturbation~$\varepsilon$. 
The dependency on~$\varepsilon$ is due to the fact that huge sets usually decompose into more convex subsets than smaller sets when they are subject to ReLU, and so, run-time increases with the size of the input set. 
Note that we compute the full reachable set of the neural network, which provides much more information than a binary (non-)robustness-certificate. The other techniques, dz, dp, dr are designed for robustness verification/attacks and do not return any further information. A run-time comparison is thus biased. Still, for smaller $\varepsilon$ and also for the free-shaped input, the absolute run-time of our methods is competitive.

Since es \cite{exactreach, julia_toolbox}, which computes the exact reachable set, requires too much time even with the smallest neural network architecture, it was not possible to conduct a meaningful comparison. 
The exact approach es only ran on the smallest neural network (iris data set, neural network with 1 hidden layer of 4 neurons) for the smallest perturbations $\varepsilon \in \{0.001, 0.005, 0.01\}$ (see Table~\ref{tab:res_julia})\footnote{Note that we used a version of \cite{julia_toolbox} in which a previously existing bug in an underlying library has been fixed. This fix is crucial for correctness, but results in longer run-times than originally reported in \cite{julia_toolbox, exactreach}.}. 
Note that the exact approach es certifies~$28$ of the~$29$ samples as robust and $0$ as non-robust and rejects one sample for which it was not able to solve an underlying optimization problem. 
\begin{table*}[ht]
	\begin{center}
		\begin{tabular}{cccccccc}
			\toprule
			$\varepsilon$ & \multicolumn{2}{c}{RsO} & \multicolumn{2}{c}{RsU}     & \multicolumn{3}{c}{es} \\
			\cmidrule(lr){2-3} \cmidrule(lr){4-5} \cmidrule(lr){6-8}
			& No. rob. & Time [ms]    & No. non-r. & Time [ms]  & No. rob.    & No. non-r. & Time [ms] \\
			\midrule
			$0.001$       & 29         & 0.47       & 0              & 0.46    & 28         & 0              & 14.68  \\
			$0.005$       & 29         & 0.46       & 0              & 0.47    & 28         & 0              & 14.56  \\
			$0.01$        & 29         & 0.47       & 0              & 0.46    & 28         & 0              & 14.61  \\
			$0.02$        & 29         & 0.58       & 0              & 2.40    & -          & -              &  $>$ 3d   \\
			\bottomrule
		\end{tabular}
	\caption{Comparison of RsO and RsU with the exact reachable set computation (es) \cite{exactreach, julia_toolbox} on 29 correctly classified samples of the iris data set (neural network with 1 hidden layer of 4 neurons, acc. $97\%$, cube setting).}
	\label{tab:res_julia}
	\end{center}
\end{table*}
When performing the exact method es on a cube-shape input with perturbation~$\varepsilon = 0.02$, it did not finish even after more than three days. This might be explained by the fact that es uses half-spaces to describe the reachable set. Applying ReLU on sets described by half-spaces requires exponential time, and thus, es is not feasible even for small neural networks.
Consequently, the reachable set needs to be over-/under-approximated as in our approach.

\textbf{Classification: Class-Specific Verification.} 
Robustness scores allow class-specific (non-)robustness verification 
in cases where distinguishing between classes is not equally important, e.g. in the tissue data set.  
The authors of the data set are of the opinion that distinguishing between the class~3, 4 and~5 (fibro-adenoma, mastopathy and glandular) is of minor importance, while it is crucial to distinguish these classes from class~$1$ (carcinoma). 
This is illustrated in Figure~\ref{fig:classspecific}, left part, where classes 3, 4 and 5 are not robust against each other, while class~1 is robust against all other classes (plot: percentage of instances which are evaluated as (non-)robust; x-axis: ground truth class, y-axis: class we test against). 
Thus, class-specific analysis allows classifiers to be evaluated more specifically and focus on crucial robustness properties.

\begin{figure*}[ht!]
	\centering
	\includegraphics[width=0.2425\textwidth]{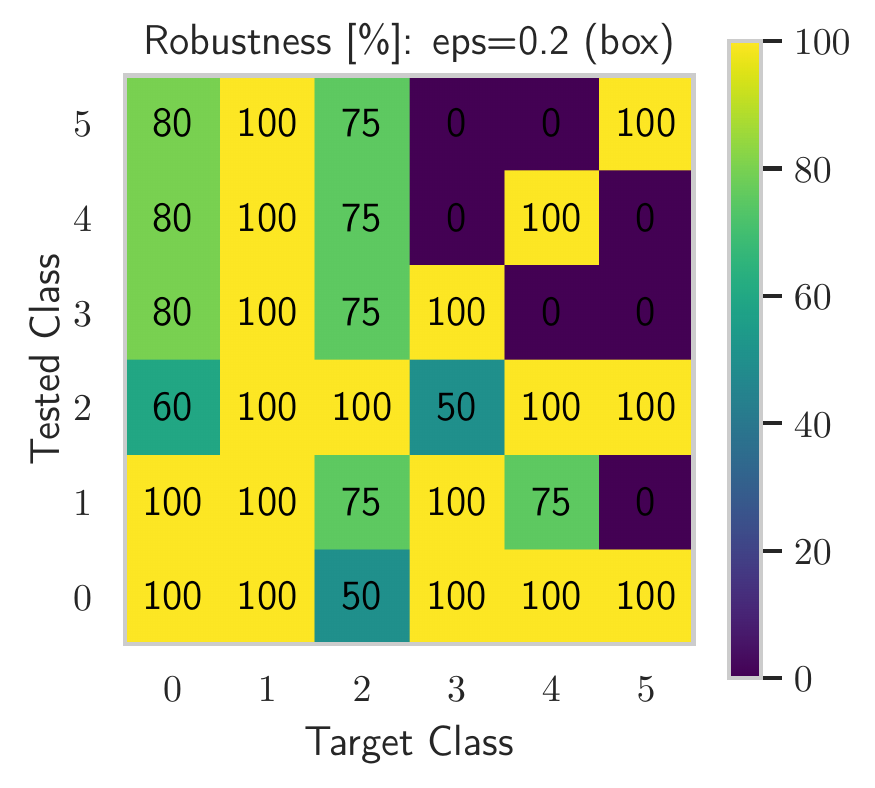}
	\includegraphics[width=0.2425\textwidth]{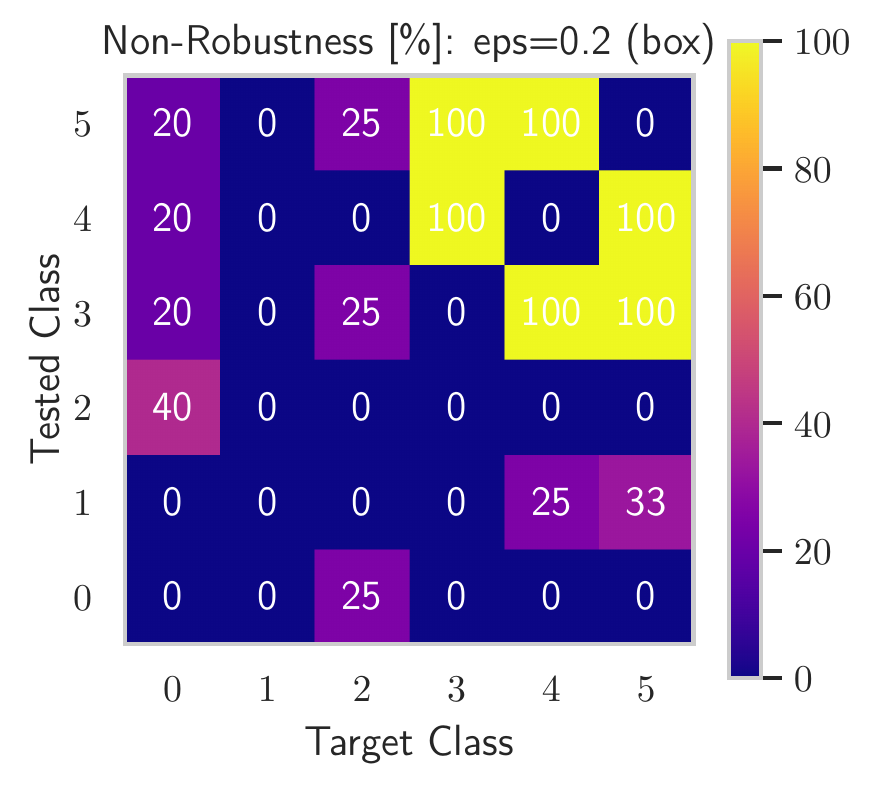}
	\rulesep
	\includegraphics[width=0.2425\textwidth]{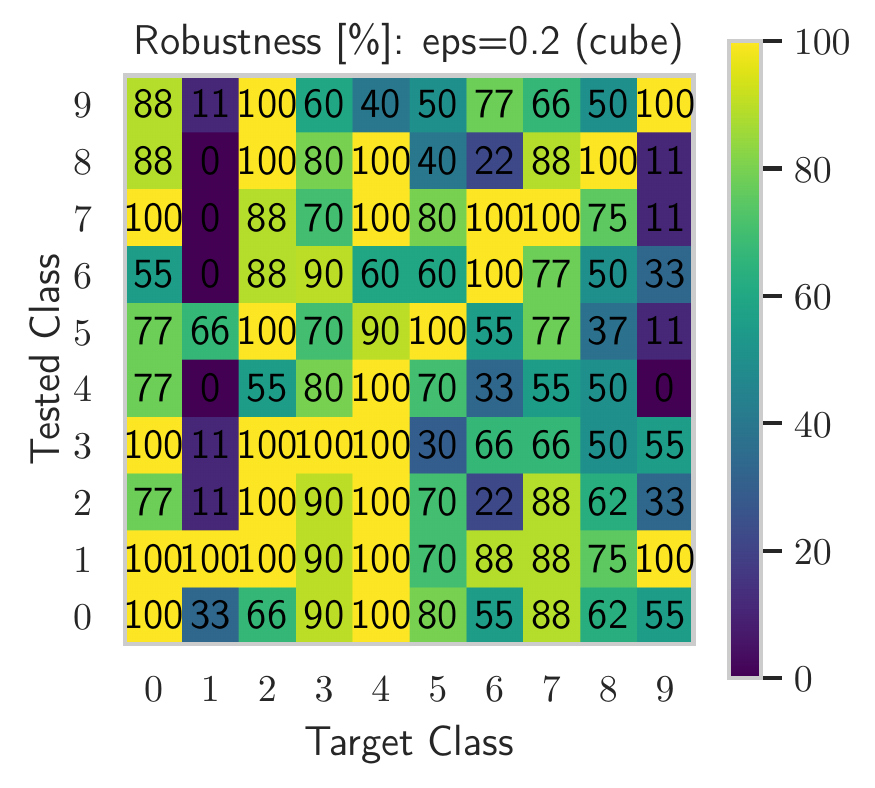}
	\includegraphics[width=0.2425\textwidth]{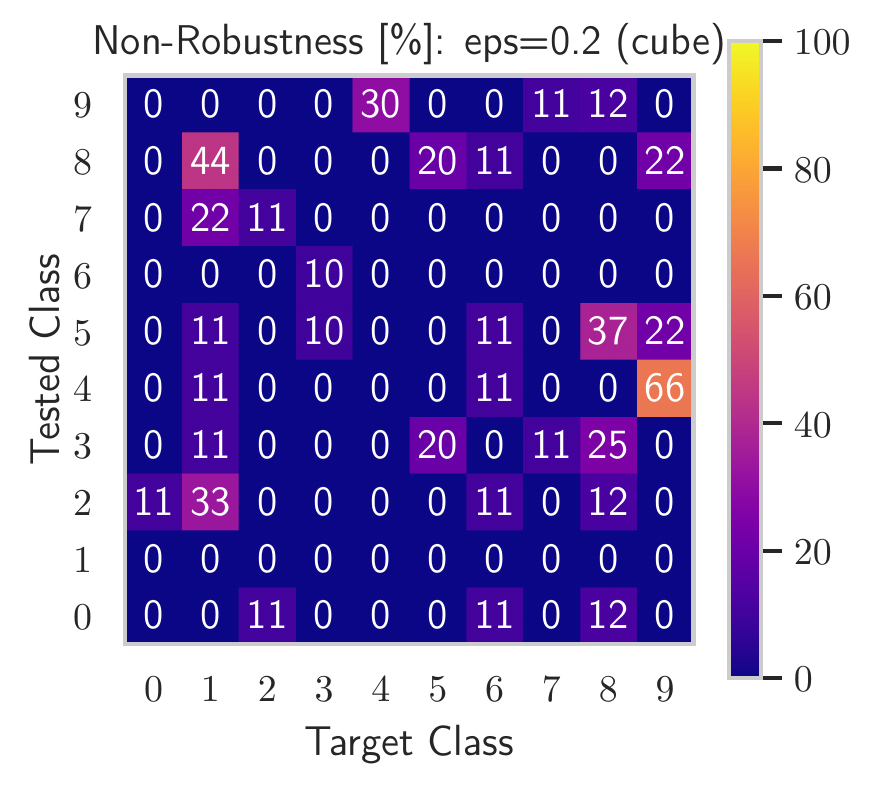}
	\caption{Class specific verification on the breasttissue data-set (3 hidden layer, acc. $97\%$, box setting, left) and Fashion-MNIST (1 hidden layer, acc. $92\%$, cube setting, right).}
	\label{fig:classspecific}
\end{figure*}

Furthermore, it allows us to draw conclusions about the concepts a neural network has learned (see Figure~\ref{fig:classspecific}, right part, Fashion-MNIST with classes: 0 top, 1 trousers, 2 pullover, 3 dress, 4 coat, 5 sandal, 6 shirt 7 sneaker, 8 bag, 9 boot). It is striking that class 2 pullover is less robust against classes of items of a similar shape (0 top, 4 coat) but robust against classes of items of different shapes  (1 trousers, 3 dress, 5 sandal, 8 bag, 9 boot). This indicates that the neural network has extracted the shape and learned its importance for a classification decision.

\textbf{Classification: Reliability of Predictions.} 
Distinguishing between reliable (label 0) and non-reliable predictions (label 1) can be seen as a binary classification problem. 
Although a wrong prediction (w.r.t.\ ground truth) can theoretically have a high robustness score, we observe that the robustness scores corresponding to wrongly predicted inputs are mostly negative or close to zero.
Thus, we consider a prediction as reliable if the corresponding robustness scores (w.r.t.\ the predicted class) is larger than a positive threshold~$\theta$.
This threshold~$\theta$ is chosen such that it maximizing the number of correctly identified reliable/ non-reliable samples on the validation set. 
Table~\ref{tab:res_reliablity} compares the performance of RsO with our proposed baseline approach that uses softmax scores to distinguish between reliable and non-reliable predictions. 

\begin{table*}[ht]
	\begin{center}
		\begin{tiny}
			\begin{tabular}{cccc}
			\toprule
				           & TPR [\%] & TNR  [\%]  & reliablity acc. [\%]  \\
			\midrule
			RsO            & $90.5$  &  $75.0$  & $89.5$ \\
			softmax scores & $92.0$  &  $71.4$  & $90.7$ \\
			\bottomrule
		\end{tabular}
	\end{tiny}
	\caption{Distinguishing between reliable and non-reliable predictions: comparison of RsO and softmax scores (fashionmnist, classification acc. $96~\%$, $\varepsilon=0.005$).}
	\label{tab:res_reliablity}
\end{center}
\end{table*}
Our comparison shows that, while softmax scores result in a slightly higher true-positive-rate and overall accuracy, 
RsO provides a significantly higher true-negative-rate. 
Thus, RsO identifies more non-reliable predictions than softmax scores. Furthermore, RsO provides a robustness certificate as well as an indicator for reliability.

\textbf{Classification: Robust Training.}
The robustness scores as defined in Equation~\ref{eq:scores} are directly used in robust training by incorporating them into the loss function, e.g. as follows: 
\begin{equation}
\begin{aligned}
L_{\mathrm{rob}} &= L_{\mathrm{pred}} + \mathbb{I}[{\mathrm{pred=target}}] \cdot \max_b \mathrm{ReLU}\left( - s_b \right)
\end{aligned}
\label{eq:loss_classification}
\end{equation}
where $L_{\mathrm{pred}}$ is the cross-entropy loss and $\mathbb{I}[{\mathrm{pred=target}}] =1$ for correctly classified inputs, otherwise~$0$. Note that the loss is fully differentiable w.r.t.\ the neural network weights (i.e. we can backpropagate through the zonotope construction) which makes it possible to train a model with enhanced robustness against any perturbation that can be described by any (input) zonotope. 
Figure~\ref{fig:res_class_robtrain} compares robustness of models obtained by robust training ($L_{{rob}}$), retraining (warm-start with a normally trained model, further training with $L_{\mathrm{rob}}$), normal training, and mixup (a robust training technique based on a convex combination of samples, see \cite{zhang2018mixup}). 
\begin{figure}[htb]
	\centering
	\includegraphics[width=0.32\textwidth]{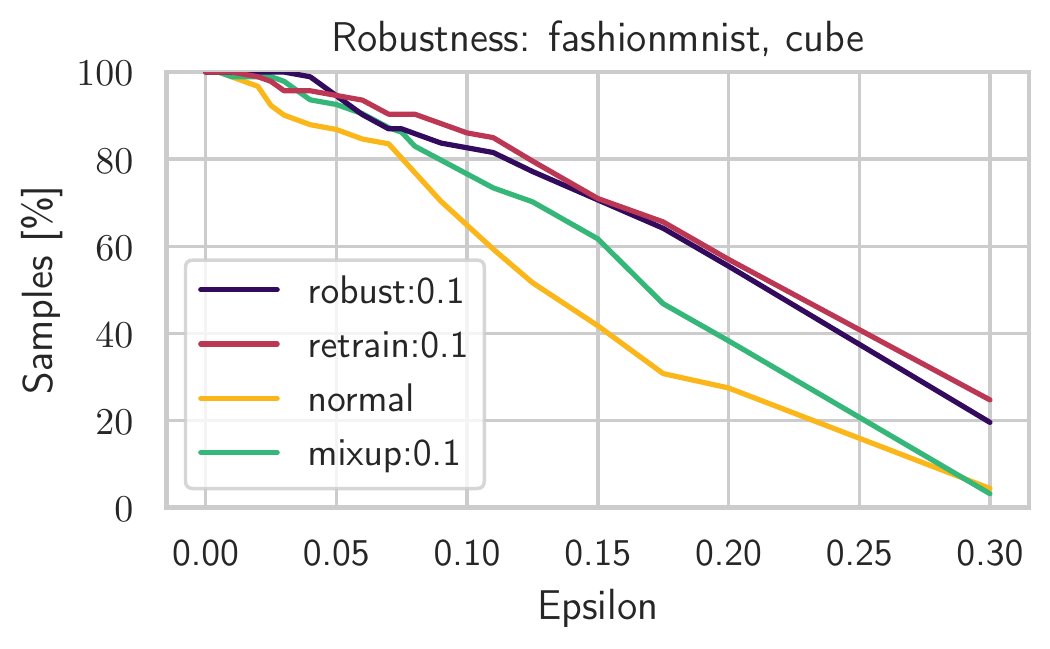} %
	\includegraphics[width=0.32\textwidth]{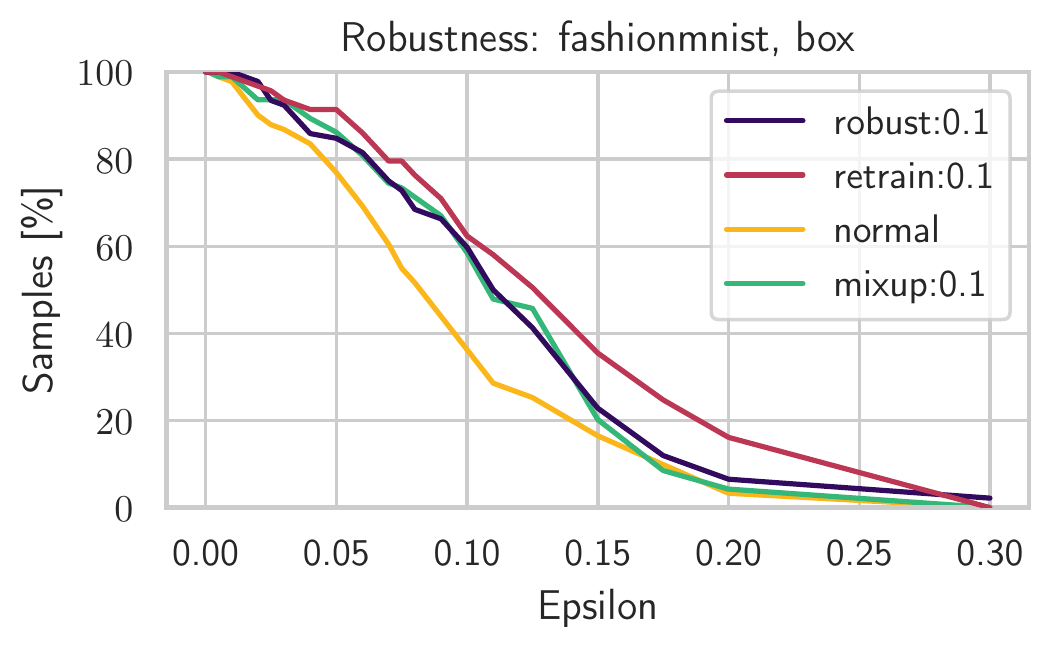} %
	\includegraphics[width=0.32\textwidth]{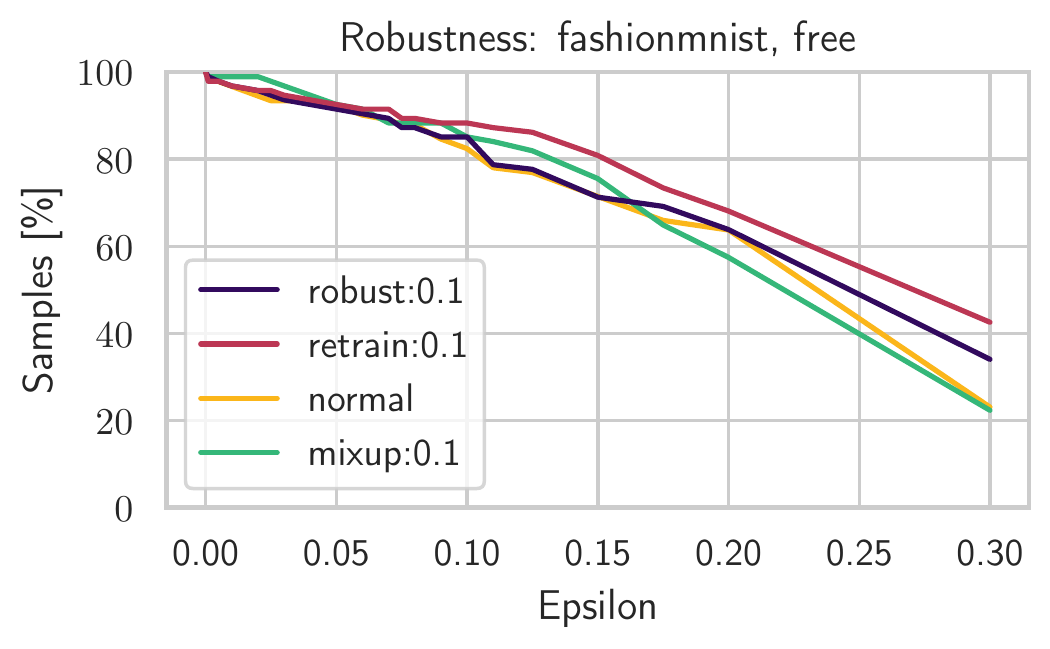} %
	\caption{Evaluation of gg against box-shaped perturbations with $\varepsilon=0.1$ on FashionMNIST (Acc.: normal $91\%$, mixup $94\%$, retrain $93\%$, robust $92\%$).}
	\label{fig:res_class_robtrain}
\end{figure}
Robust training, retraining and mixup enhance the robustness of the neural network on cube-, box- and free-shaped perturbations as well as the accuracy of the neural network. While the performance of mixup and robust training are comparable on box- and free-shaped perturbations, retraining outperforms mixup on all three perturbation shapes.

\textbf{Regression: (Non-)Robustness Analysis and Robust Training.}
Obtaining robust neural networks is desirable in any task but has mainly been studied for the purpose of classification. 
Classifiers are robust if an input~$x$ and all points in its neighborhood are assigned to the same label. 
In regression tasks, there is no equivalent robustness definition, because outputs are continuous and not categorical. 
However, intuitively, regression models are robust if close inputs result in close outputs. 
Assume that inputs and outputs are standardized before training, such that all features are on an equal scale. 
The extension~$l_a$ of output feature~$a$ within the reachable set~$R_S$ quantifies robustness: the smaller~$l_a$ is, the more robust is the model. 
The extension is defined by the two most distant points~$u$ and~$v$ within~$R_S$ w.r.t. dimension~$a$:
$l_a = \left| \max_{u \in R_S}  u_a - \min_{v \in R_S} v_a \right|$.
For input features, the extension~$l_{\mathrm{in}}$ is equivalently defined on the input set. In the cube setting,~$l_{\mathrm{in}}$ is the same for all input features.
 
If we have $l_a \leq l_{\mathrm{in}}$ for all output features~$a$, the regression model maps close inputs to close outputs and we consider it as robust. 
We use this robustness definition to define a robust training function based on feature extension and a standard loss function~$L_{\mathrm{val}}$ (e.g. Huber loss): 
\begin{equation}
\begin{aligned}
	L_{\mathrm{rob}} &= L_{\mathrm{val}} + \mathrm{ReLU} \left( \max_{a} l_a - l_{\mathrm{in}} \right)
\end{aligned}
\label{eq:regression_loss}
\end{equation}
If $l_a$ is larger than~$l_{\mathrm{in}}$ the second term of~$L_{\mathrm{rob}}$ is positive, otherwise it is zero. 
We compare four different training modes: normal (training with Huber loss), retrain (warm-start with a normally trained model, and further training with $L_{\mathrm{rob}}$), robust (training with $L_{{rob}}$), and mixup (a training technique that convexly combines inputs, see \cite{zhang2018mixup}). 
Figure~\ref{fig:res_abalone_2_0_075} illustrates the training and robustness analysis, based on the abalone data set (2 hidden layers, first row) and the airfoil dataset (1 hidden layer, second row). 
\begin{figure}[ht]
	\centering
	\includegraphics[width=0.32\textwidth]{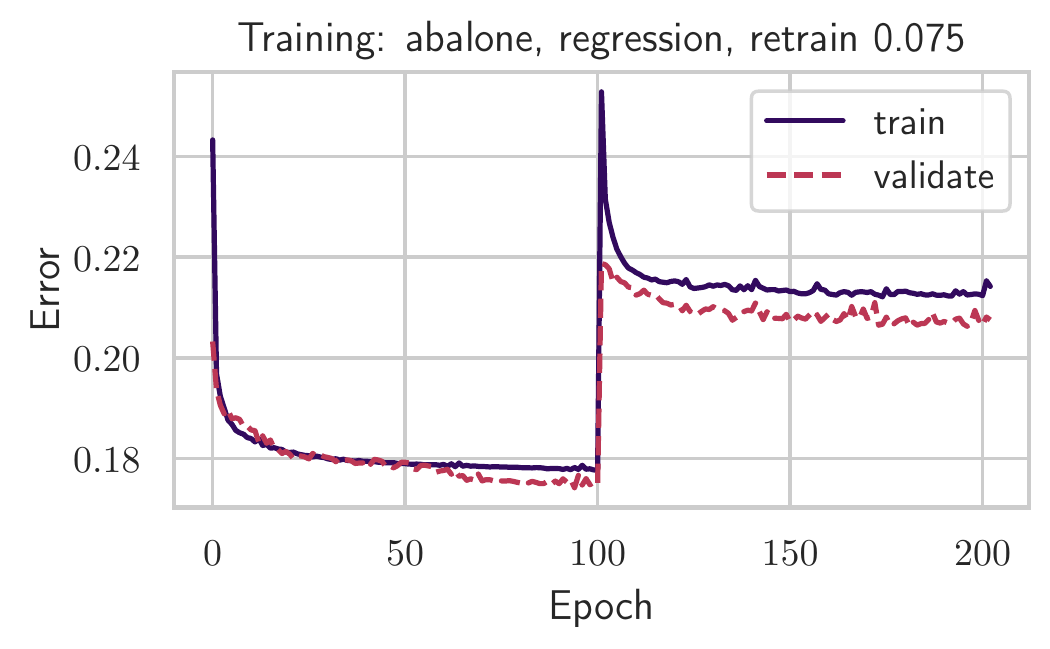}
	\includegraphics[width=0.32\textwidth]{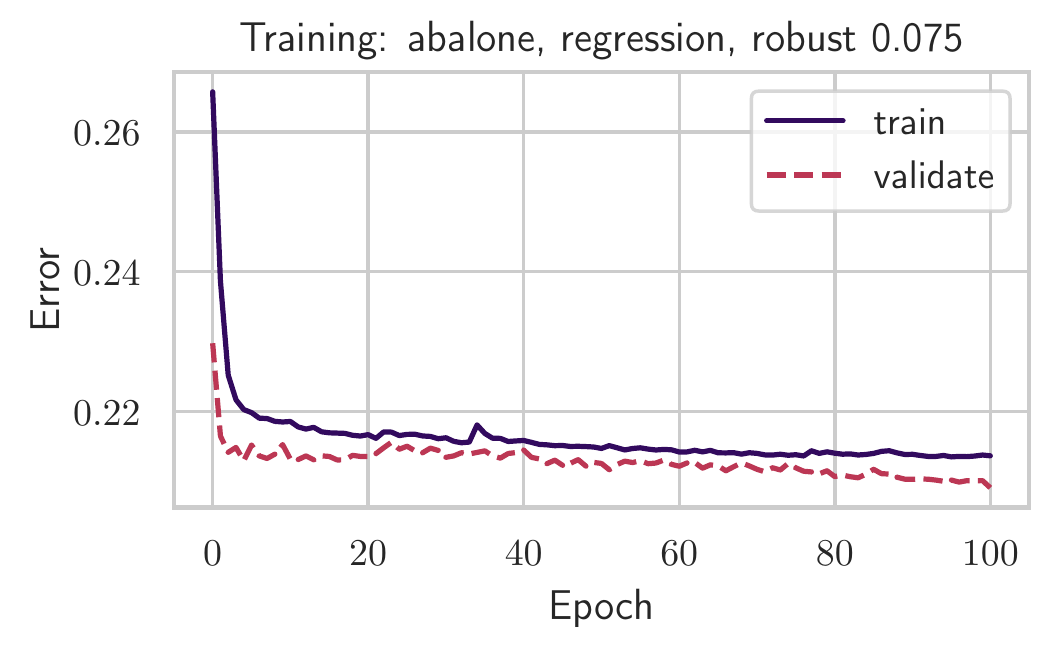}
	\includegraphics[width=0.32\textwidth]{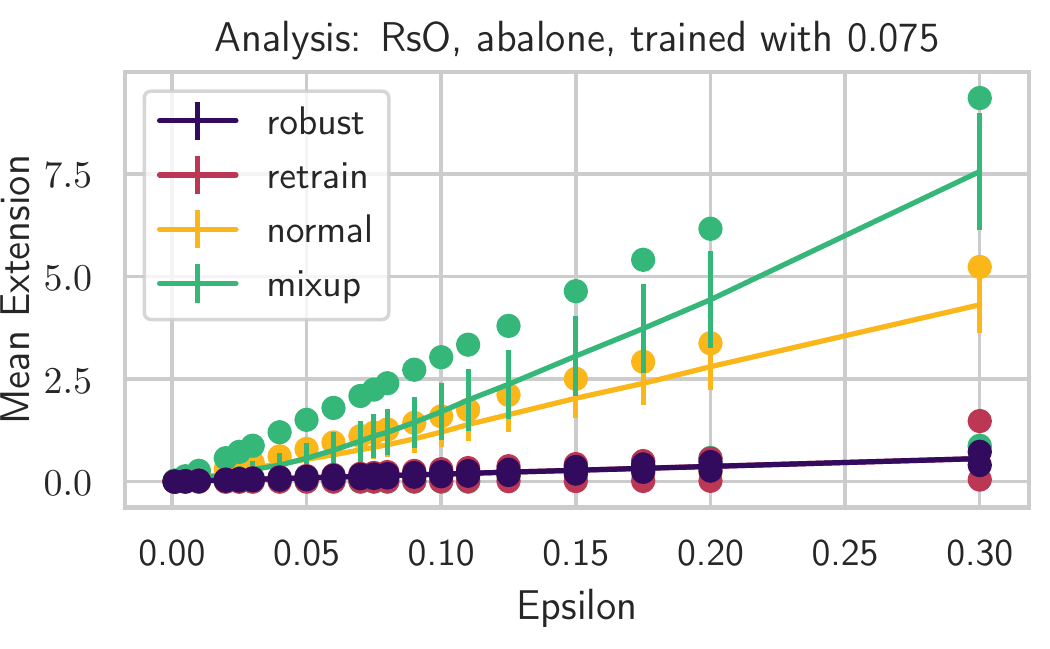}
	\includegraphics[width=0.32\textwidth]{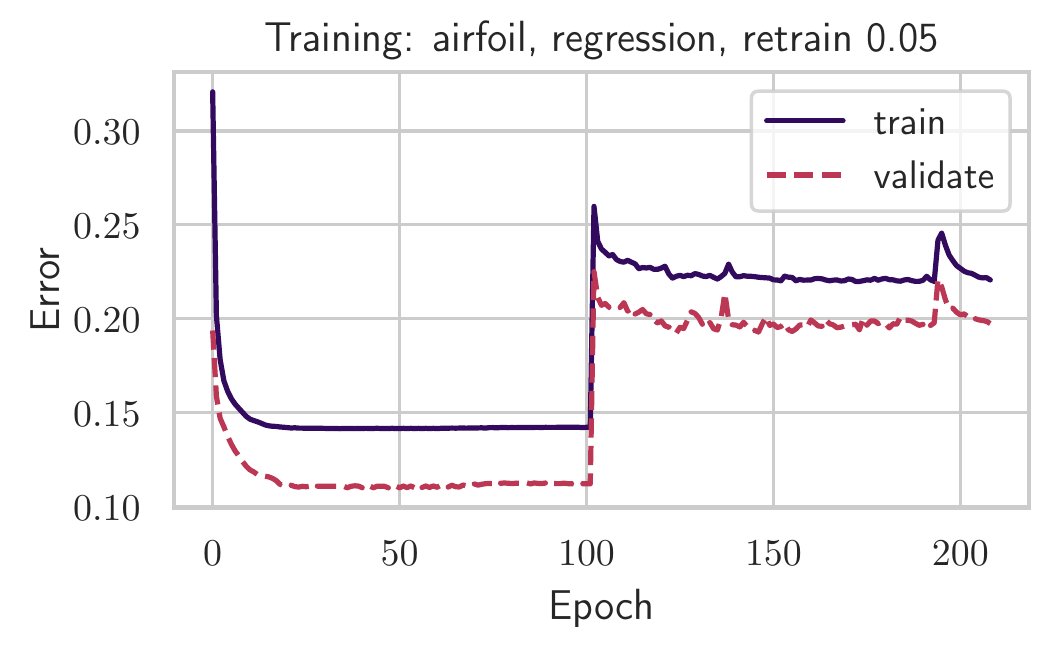}
	\includegraphics[width=0.32\textwidth]{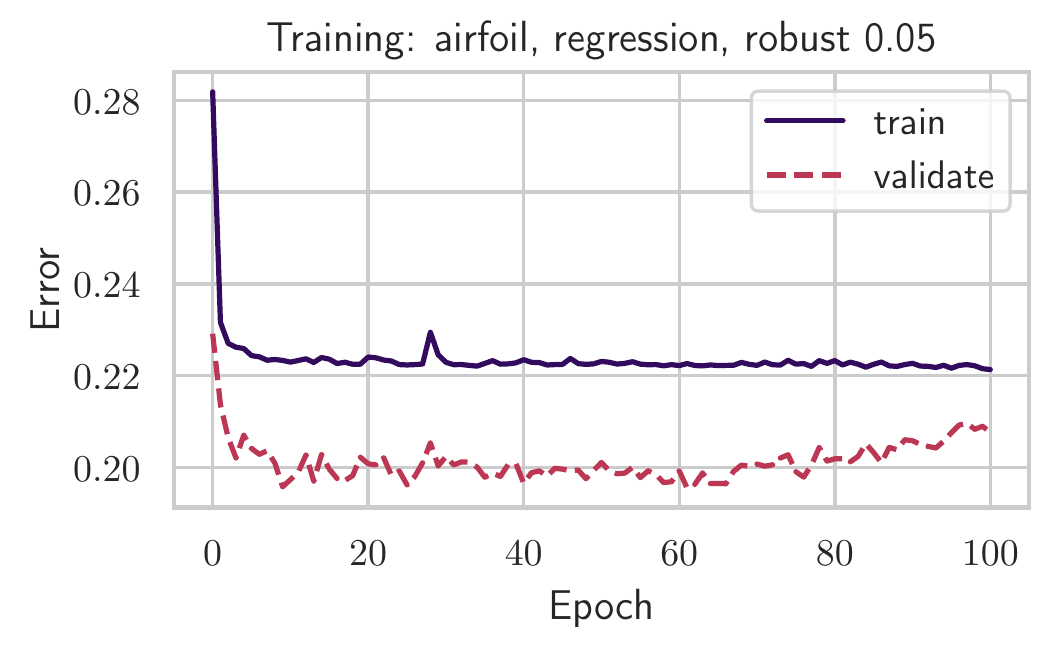}
	\includegraphics[width=0.32\textwidth]{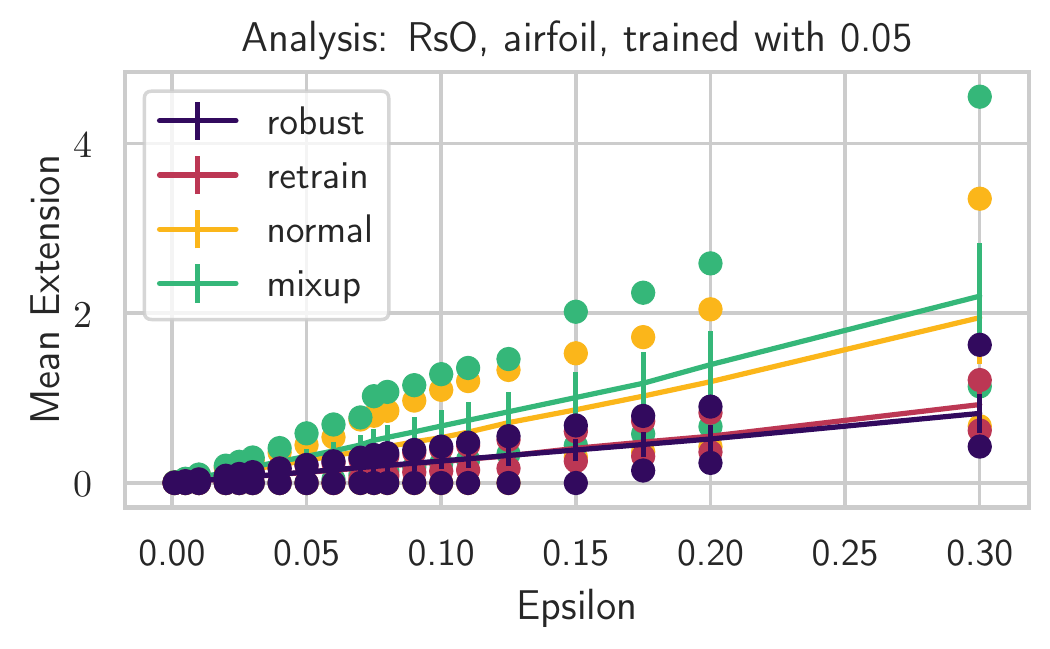}
	\caption{Robust training and robustness analysis of regression models. The smaller the mean extension the more robust is the model. (Error on the test set -- abalone data: normal $0.20$, retrain $0.24$ (start at epoch~$100$), robust $0.24$, mixup $0.20$, airfoil data: normal $0.12$, retrain $0.18$, robust $0.18$, mixup $0.11$).}
	\label{fig:res_abalone_2_0_075}
\end{figure}
While mixup seems to decrease the robustness of regression models, robust training and retraining results in smaller reachable sets and thus ensures that close inputs are mapped to close outputs. 
Thus, robust training and retraining both improve robustness properties without significantly reducing prediction accuracy.

\textbf{Explainability: Feature Ranking for Classifiers \& Regression Models.}
Reachable sets enable the importance of features to be quantified w.r.t.\ a model output. 
To quantify the influence of feature~$f_1$, we define a box-shaped input set with a large perturbation $\delta$ on~$f_1$, while the perturbation on the remaining features is small.
The size of the reachable set corresponding to~$\hat{Z}_{f_1}$ captures the variation in the predictions caused by varying~$f_1$ and thus quantifies the influence of~$f_1$. 
Since the exact size/volume of~$\hat{Z}_{f_1}$ is inefficient to compute~\cite{gover2010}, we approximate it using the interval hull.  
Here, we use the scaled version of the volume that considers the dimensionality~$d$ of the zonotope: 
$V \left(\mathrm{IH} (Z)\right) = \left( 2 \prod_i \delta g_i \right)^{\frac{1}{d}}$
where $\delta g = \sum_{i} |g_i|$.
The volume of the reachable set is approximated by the sum of all interval hull volumes. 
\begin{figure}[ht]
	\centering
	\includegraphics[width=0.32\textwidth]{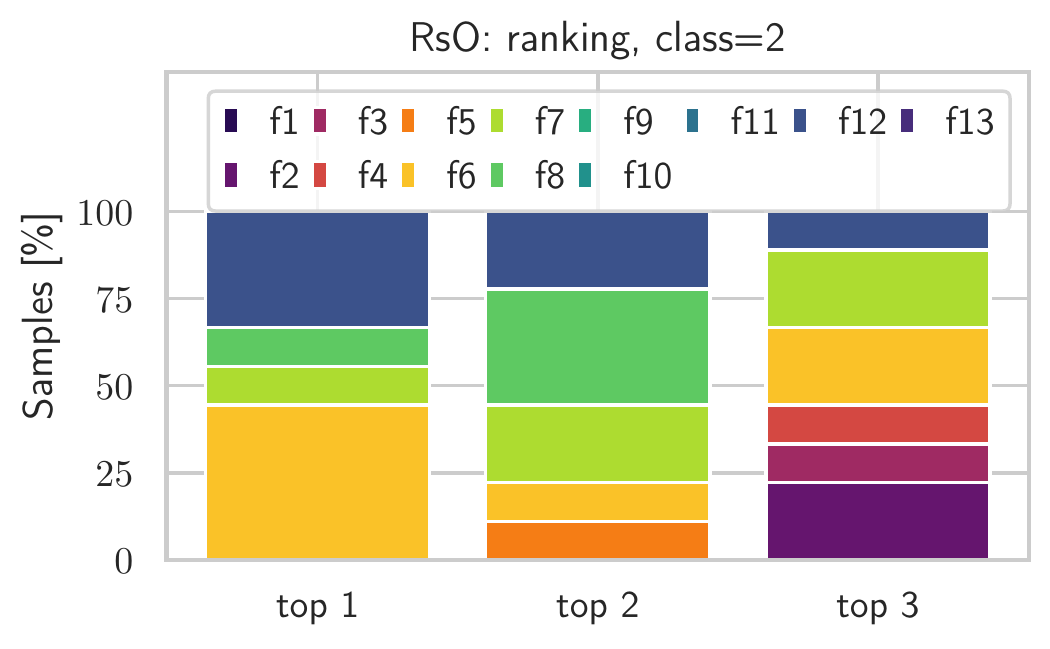}
	\includegraphics[width=0.32\textwidth]{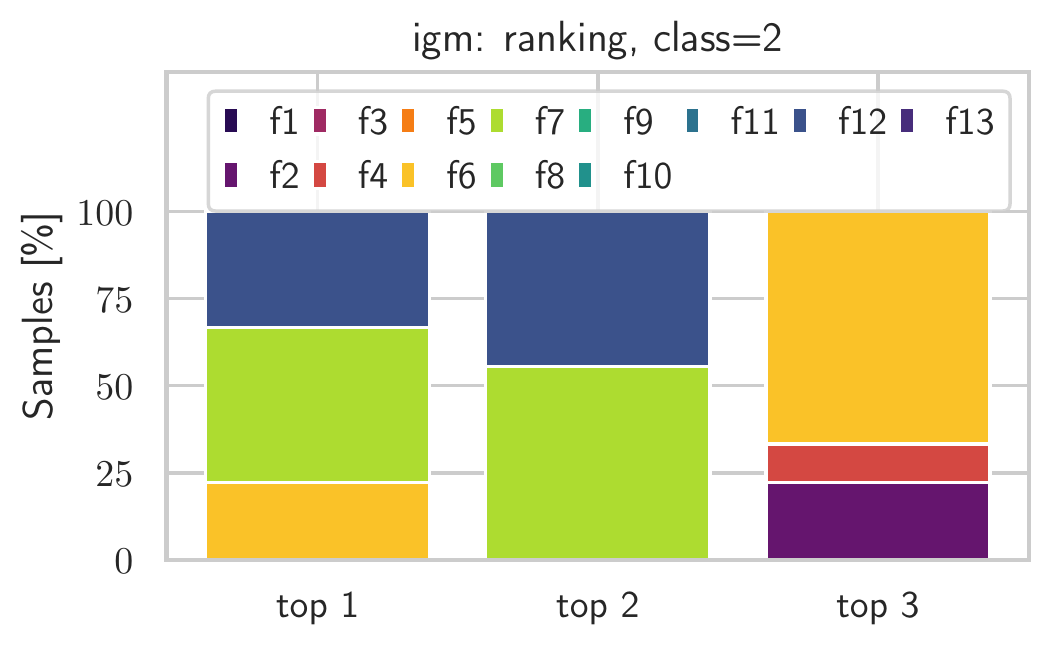}
	\includegraphics[width=0.32\textwidth]{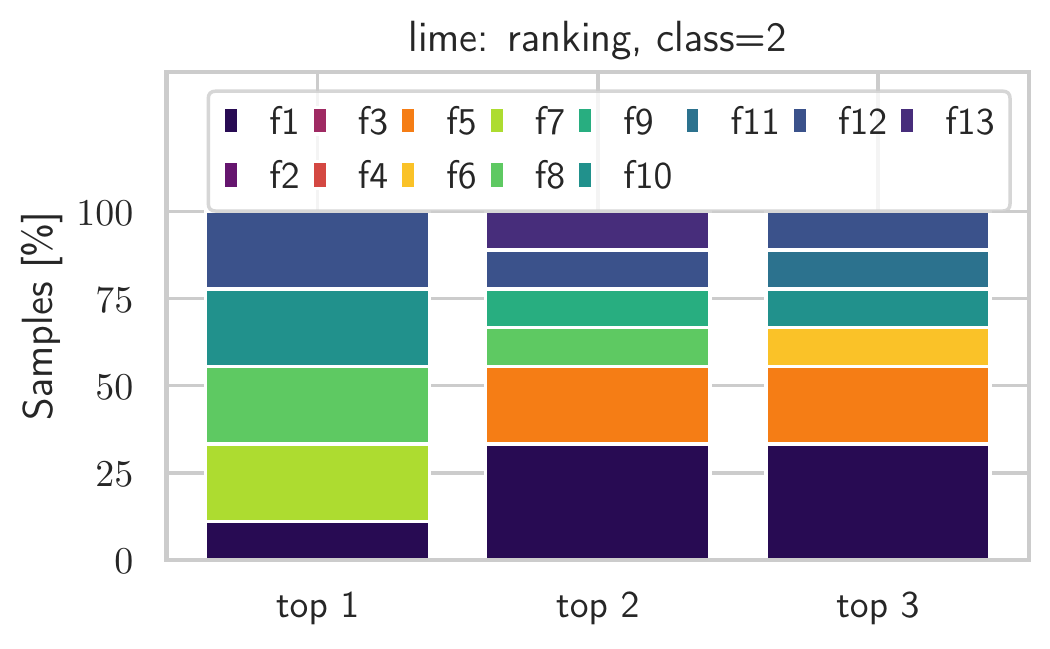}
	\caption{Ranking (top $3$ features) for samples (y-axis) of the wine data set (class 2, 13 features, 1 hidden layer, neural network accuracy $93.3~\%$) computed by RsO (left), igm (middle) and lime (right).}
	\label{fig:res_wine_explain}
\end{figure}
	Figure~\ref{fig:res_wine_explain} illustrates the three most important features for samples of the wine data set 
	computed by our RsO approach in comparison with two other approaches: the integrated gradients method (igm) \cite{sundararajan2017igm} and local interpretable model-agnostic explanations (lime) \cite{lime}. 
	RsO identifies four possibilities for the most important feature: $f12$ (blue, $\approx 30\%$ of samples), $f9$ (teal, $\approx 10\%$ of samples), $f7$ (bright green, $\approx 50\%$ of samples) and $f6$ (yellow, $\approx 20\%$ of samples). 
	Igm identifies three of these possibilities, while lime identifies five possibilities. 
	Overall, the rankings of RsO, igm and lime are of different complexity in terms of different features. The most (second-most/third-most) important feature identified by igm adopts 2-3 possibilities, by lime 5-6 possibilities and by RsO 4-6 possibilities. Consequently, the complexity of the feature ranking computed by RsO is between the one obtained by igm and lime.

\textbf{Reachable Set Approximation: Analysis of the Limits.}\label{sec:limits}
RsU and RsO approximate the reachable set of a ReLU network layer-by-layer. Within each layer, they compute a linear transformation (defined by the weights and biases) and approximate the outcome of applying ReLU by a set of convex subsets (zonotopes). The number of subsets required to approximate ReLU(Z) is the bottleneck of our approaches. Worst case, applying ReLU on $Z = (c \mid G)$, $c \in \mathbb{R}^D$, $G \in \mathbb{R}^{n \times D}$ results in $2^D$ subsets/zonotopes.
Considering a neural network with $K$ layers of $D_1, D_2, D_3, \dots D_K$ neurons, the reachable set approximation requires up to 
$2^{\sum_{k=1}^{K} D_k}$
subsets/zonotopes. 
Figure~\ref{fig:res_limits} illustrates that the run time of RsO linearly increases with the number of subsets and thus exponentially increases with the number of neurons in the worst case. Thus, the number of neurons limits the applicability of our approaches on large neural networks. 
\begin{wrapfigure}{r}{0.45\textwidth}
	\centering
	\resizebox{0.45\textwidth}{!}{
		\includegraphics[width=0.4\textwidth]{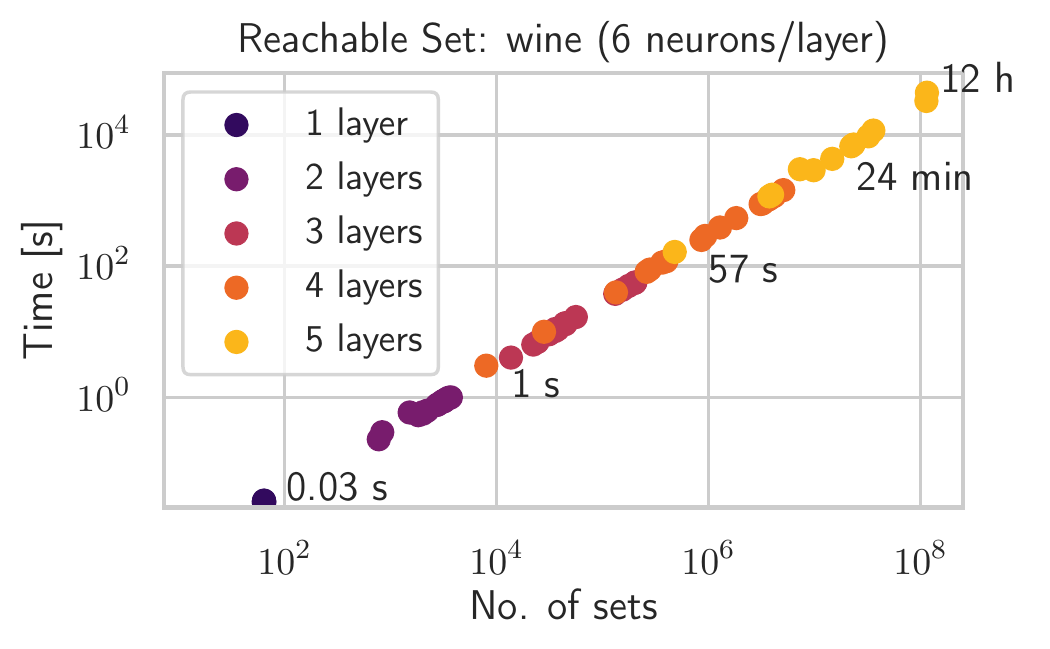}
	}
	\caption{
		Run time of RsO vs. worst case no. of subsets required to approximate the reachable set (wine data set, acc. 1,2 layers: $93\%$, acc. 3-5 layers: $97\%$).
	}
	\label{fig:res_limits}
\end{wrapfigure}
To improve this, we propose an extension, which restricts the amplification number and the total number of zonotopes (see \hyperlink{balancing}{Section Balancing approximation tightness and run time}), which is applicable to larger neural networks. Results on this extension on robustness verification and for the analysis of autoencoders are presented in the next section and in the appendix (see Section~\ref{sec:larger_nn}).

\textbf{Autoencoder Analysis.}
To illustrate the strength of our approach, we compare reachable sets obtained by RsO and RsU with a sampling-based set approximation. We approximate the reachable set of an autoencoder (three hidden layers, $60 \times 30 \times 60$ neurons) with respect to a cube shaped input set with~$\varepsilon=0.001$. 
RsO and RsU are restricted such that the maximum amplification of a zonotope is~$A=100$ and the overall number of zonotopes is less or equal to~$B=1000$. 
To compare with RsO and RsU we introduce a simple baseline based on sampling. 
This sampling approach chooses~$10^9$ points among the vertices of the cube shaped input set and computes the corresponding outputs. The set spanned by these~$10^9$ outputs is used to approximate the exact reachable set.

\begin{figure*}[ht!]
	\flushright
	\includegraphics[width=0.24\textwidth]{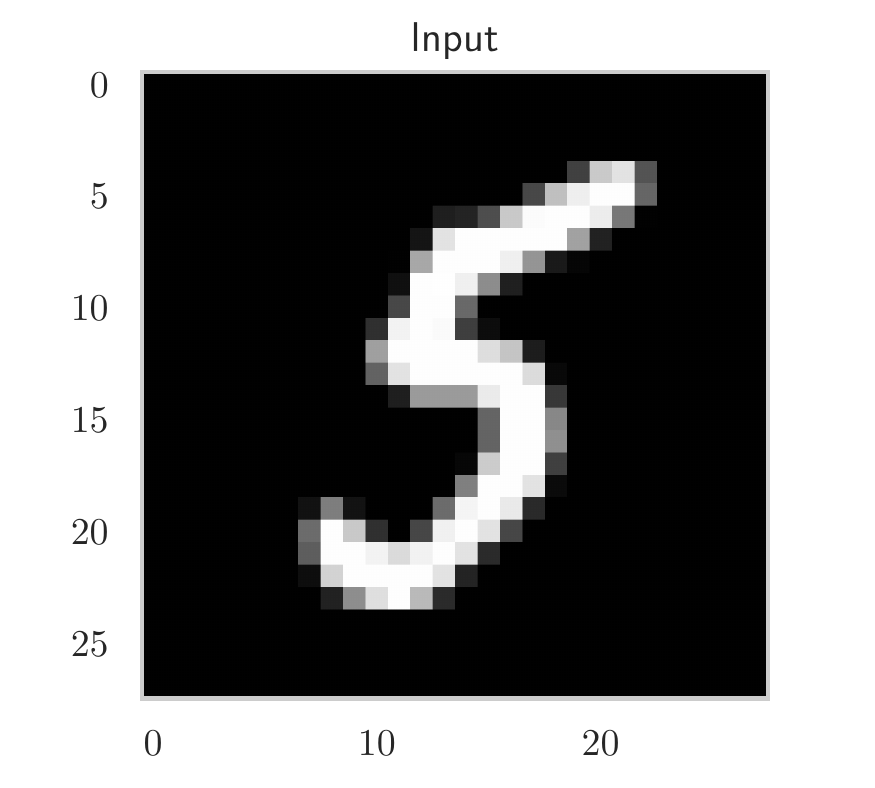}
	\includegraphics[width=0.24\textwidth]{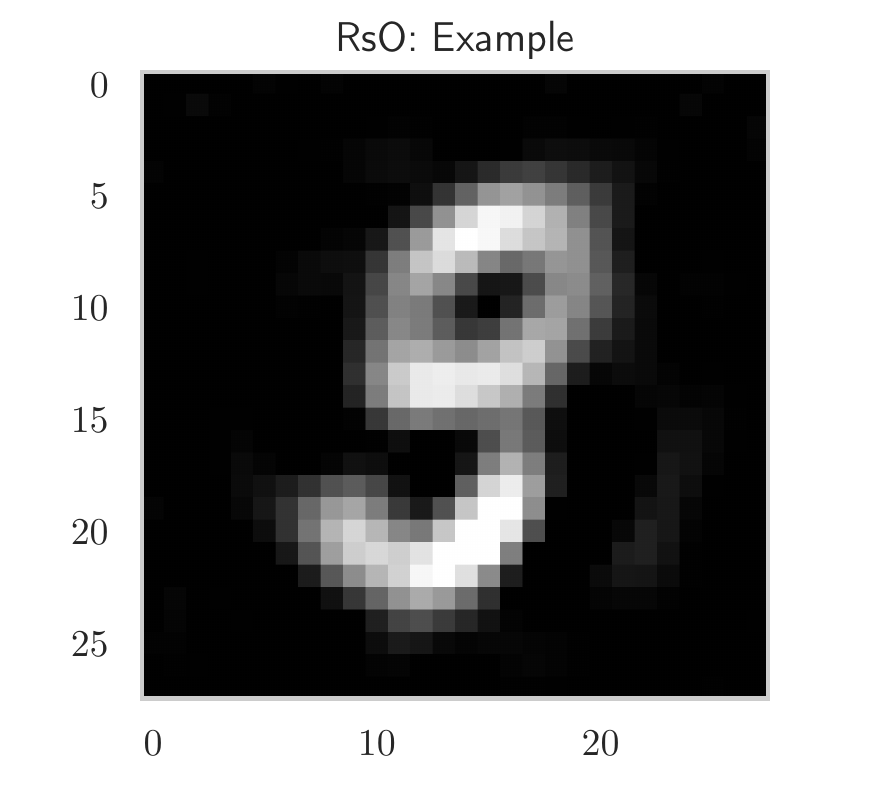}
	\includegraphics[width=0.24\textwidth]{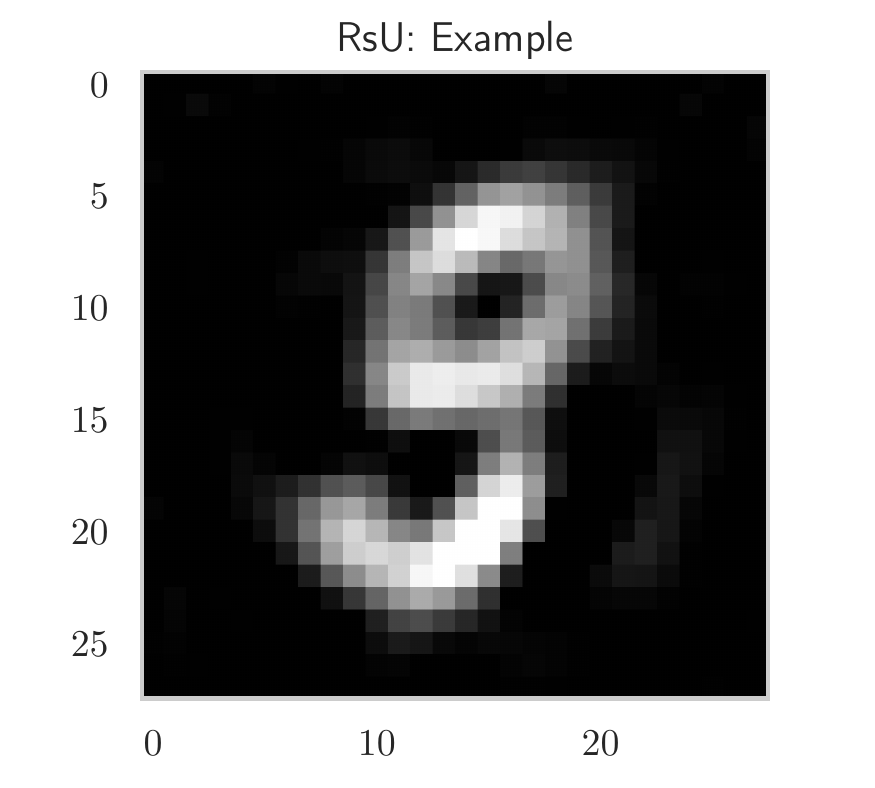}
	\includegraphics[width=0.24\textwidth]{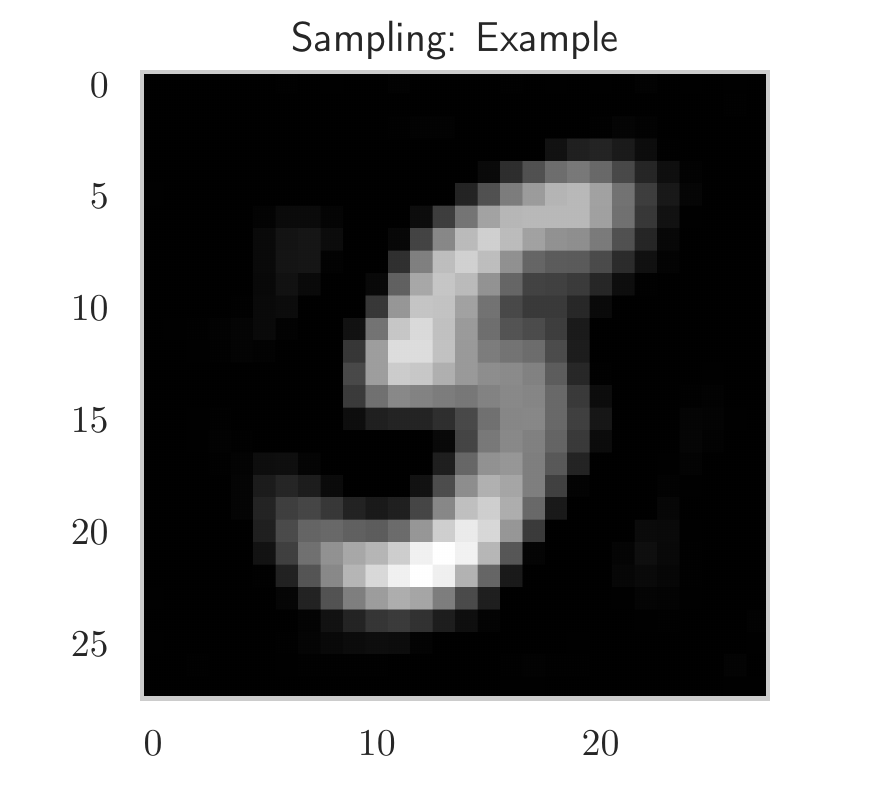}
	\includegraphics[width=0.235\textwidth]{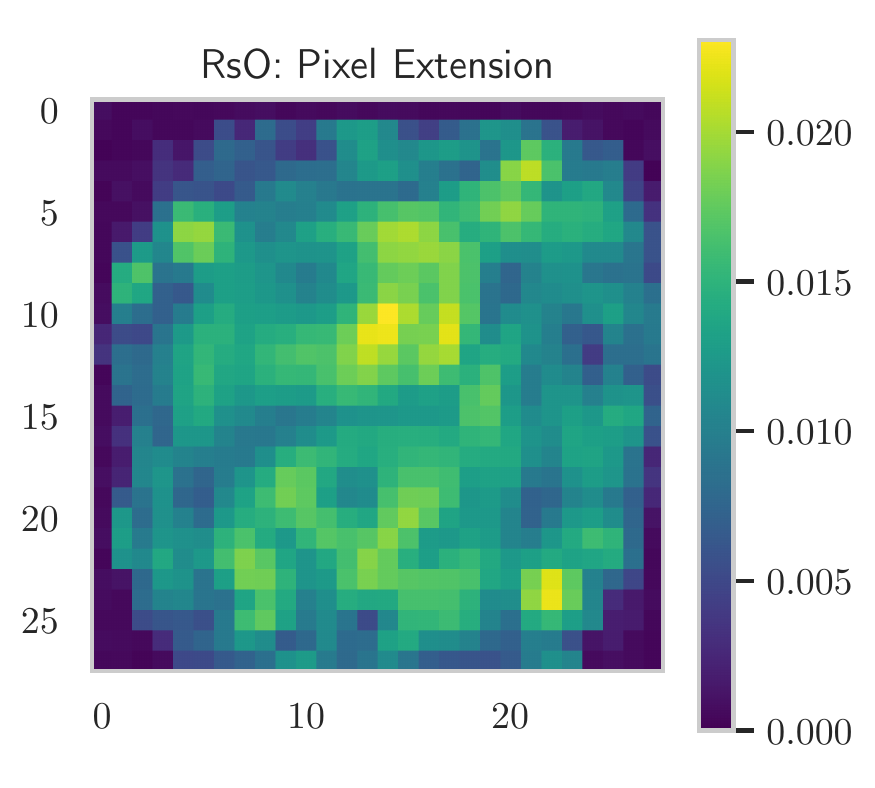}
	\includegraphics[width=0.235\textwidth]{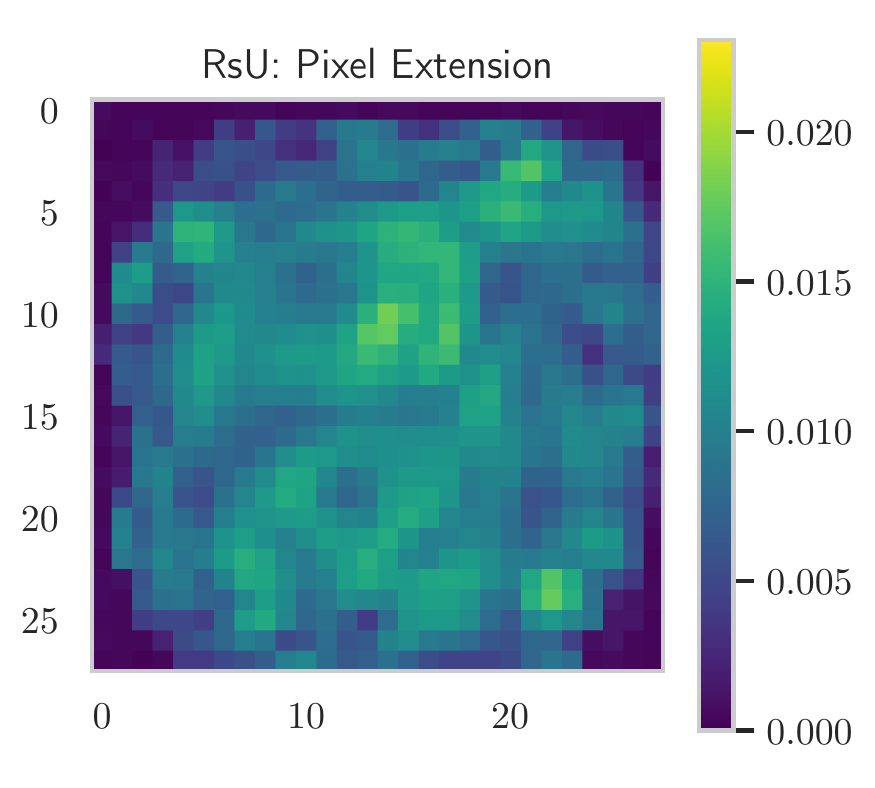}
	\includegraphics[width=0.235\textwidth]{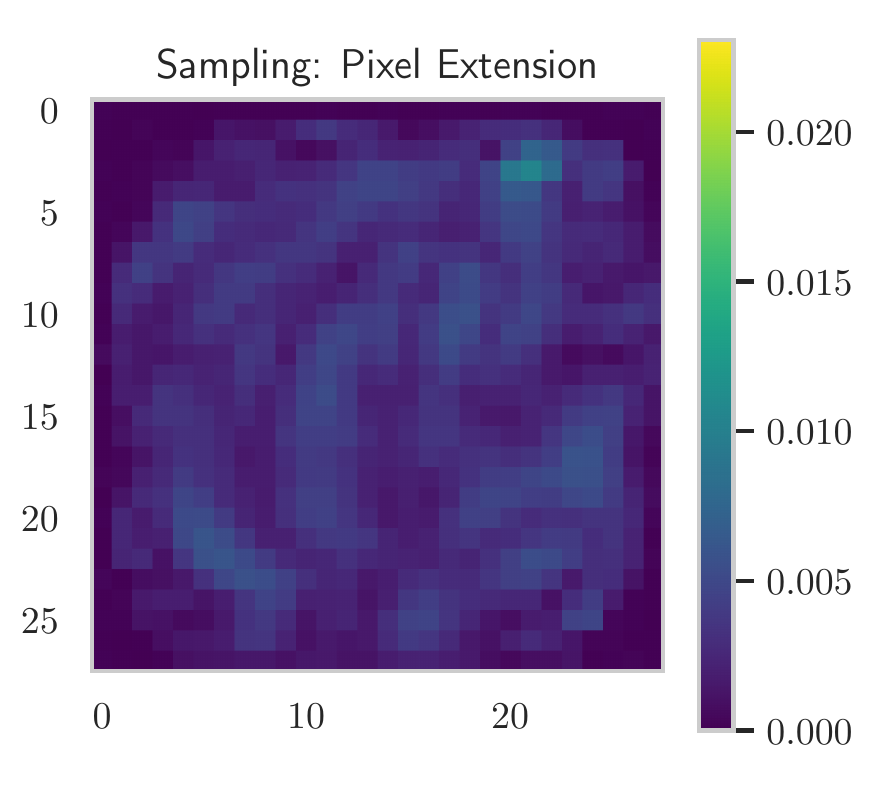}
	\caption{Analysis of an autoencoder (MNIST, cube, $\varepsilon=0.001$). First row: input image, output image drawn from the reachable set approximated by RsO (second), RsU (third) and sampling (fourth). Second row: extension/size of the pixel range corresponding to the reachable sets computed by RsO (left), RsU (middle) and sampling (right). The smaller the ranges computed by RsO and the larger the ranges computed by RsU or sampling the better is the performance.}
	\label{fig:autoencoder_mnist}
\end{figure*}

Since we consider autoencoder models, the reachable set consists of pictures from the same space as the input. 
To visualize the properties of the reachable sets computed by RsO, by RsU and by the sampling approach, we draw example pictures from the reachable sets. Furthermore, we compute the extension/size of the range of each pixel based on the reachable set under consideration (Figure~\ref{fig:autoencoder_mnist}).

Even though we restrict the number of convex subsets to~$1000$, RsO and RsU result in similar example pictures and similar extensions for each pixel (see Figure~\ref{fig:autoencoder_mnist}, second row and third row). 
This illustrates that our approximations are tight and close to the exact reachable set, \textit{since the exact reachable set is enclosed by the under- and over-approximation.} 
In comparison to RsU, the sampling approach results in pixel extensions that are about two times smaller/worse and example pictures that are too close to the image reconstructed from the original input. 
Thus, sampling $10^9$ instances from the input set and computing the corresponding outputs still leads to a dramatic underestimation of the exact reachable set. 
This shows that RsU outperforms the sampling approach, even if we restrict the overall number of zonotopes and the possible amplification. 
In conclusion, these results highlight the fact that computing an upper bound (RsO) and a lower bound (RsU) to the reachable set of neural networks provides more information on the mapping of networks than sampling.

\section{Conclusion}
\label{sec:conclusion}

We propose RsO and RsU as two efficient approaches for over- and under-approxi\-mating the reachable sets of ReLU networks. 
Approximated reachable sets are applicable to the analysis of neural network properties: 
we analyze and enhance the (non-)robustness properties of both classifiers and regression models. 
Our approach outperforms PGD attacks as well as state-of-the-art methods of verification for classifiers with respect to non-norm bound perturbations. 
Reachable sets provide more information than a binary robustness certificate. We use this information for class-specific verification, robustness quantification, robust training,  distinguishing between reliable and non-reliable predictions, ranking features according to their influence on a prediction and analyze autoencoders.

\section*{Acknowledgements}

This research was supported by BMW AG. 
We would like to thank Marten Lienen for help with the toolbox that was used to compute the exact reachable set.

\bibliographystyle{plain}
\bibliography{literatur}

\newpage
\section{Appendix}
\label{sec:appendix}

\subsection{Pseudocode}
\label{sec:pseudocode}

Algorithm~\ref{algo:rso} and~\ref{algo:rsu} show how we under-/over-approximate the outcome of applying ReLU on a zonotope, while Algorithm~\ref{algo:propoagate_nn} and~\ref{algo:propoagate_nn_limit} show how the reachable set of a neural network is approximated with and without limitations on the number of used subsets.

\begin{algorithm}
	\KwIn{Zonotope $Z = (c \mid G)$, Maximum number MaxAmp of subsets used to approximate one zonotope }
	\KwOut{Set of zonotopes $RS = \{\hat{Z}_n\}_n$ that over-approximates ReLU(Z)}
	Compute index sets $R_n$, $R$ (see Equation~\ref{eq:index_sets})\;
	Project $Z$: $\forall i, \forall d \in R_n: c_d = 0$ and $G[i,d] = 0$\;
	Compute quadrants with $S_k \subseteq Z$: $\{R_k\}_k = \mathcal{P} (R) =$ power set of $R$\;
	Initialize $RS = \{\}$\;
	\If{$|\mathcal{P}(R)| >$ MaxAmp}{
		Compute interval hull $\mathrm{IH}\left(Z\right) := \left[c - \sum_{i} |g_i|, c + \sum_{i} |g_i| \right]$\;
		Restrict $\mathrm{IH}$ to its positive parts \;
		$Z_{\mathrm{IH}} = (c | G_{\mathrm{IH}})$ with $G_{\mathrm{IH}} = \mathrm{diag} (\sum_{i} |g_i|)$ \;
		$RS = \{ \mathrm{Z_{\mathrm{IH}}} \}$ \;
	}
	\Else{
		\For{$R_k \in \mathcal{P}(R)$}{
			Overapproximate $S_k$ by $\hat{Z}_k$ (see Equation~\ref{eq:over_zono})\;
			$RS = RS \cup \{\hat{Z}_k\}$
		}
	}
	\Return{$RS$}\;
	\caption{{\sc RsO} over-approximates applying ReLU on a zonotope}
	\label{algo:rso}
\end{algorithm}

\begin{algorithm}
	\KwIn{Zonotope $Z = (c \mid G)$, Maximum number MaxAmp of subsets used to approximate one zonotope }
	\KwOut{Set of zonotopes $RS = \{\hat{Z}_n\}_n$ that under-approximates ReLU(Z)}
	Compute index sets $R_n$, $R$ (see Equation~\ref{eq:index_sets})\;
	Project $Z$: $\forall i, \forall d \in R_n: c_d = 0$ and $G[i,d] = 0$\;
	Compute quadrants with $S_k \subseteq Z$: $\{R_k\}_k = \mathcal{P} (R) =$ power set of $R$\;
	Initialize $RS = \{\}$\;
	\For{$R_k \in \mathcal{P}(R)$}{
		Underapproximate $S_k$ by $\hat{Z}_k$ (see Equation~\ref{eq:under_opt})\;
		$RS = RS \cup \{\hat{Z}_k\}$ \;
		\If{$|RS| >$ MaxAmp}{
			break \;
		}
	}
	\Return{$RS$}\;
	\caption{{\sc RsU} under-approximates applying ReLU on a zonotope}
	\label{algo:rsu}
\end{algorithm}

\begin{algorithm}
	\KwIn{Zonotope $Z^0 = (c^0 \mid G^0)$, approximation method (RsO or RsU)}
	\KwOut{Set of zonotopes $RS = \{Z_n\}_n$ that approximates the reachable set}
	Initialize set of zonotopes $RS = \{ Z^0\}_n$ \\
	\For{$k \gets 1$ \textbf{to} $K$ \tcp{iterate over layers}}{ 
		$RS' = \{\}$ \;
		$RS'' = \{\}$ \;
		\For{$Z \in RS$}{
			Linear transformation: $Z' = \mathrm{lintrans}(Z)$ (see Equation \ref{eq:lintrans}) \;
			$RS' = RS' \cup \{Z'\}$ \;
		}
		\If{over approximate} {
			\For{$Z \in RS'$}{
				Apply ReLU activation function: $RS'' = RS'' \cup \mathrm{RsO}(Z, \infty)$  \;
			}
		}
		\If{under approximate} {
			\For{$Z \in RS'$}{
				Apply ReLU activation function: $RS'' = RS'' \cup \mathrm{RsU}(Z, \infty)$ \;
			}
		}
		$RS = RS''$ \;
	}
	\Return{$RS$}\;
\caption{{\sc PropZ} propagates zonotope through ReLU network}
\label{algo:propoagate_nn}
\end{algorithm}

\begin{algorithm}
	\KwIn{Zonotope $Z^0 = (c^0 \mid G^0)$, approximation method (RsO or RsU), Maximum number of zonotopes MaxZono, Maximum amplification MaxAmp}
	\KwOut{Set of zonotopes $RS = \{Z_n\}_n$ that approximates the reachable set}
	Initialize set of zonotopes $RS = \{ Z^0\}_n$ \\
	\For{$k \gets 1$ \textbf{to} $K$ \tcp{iterate over layers}}{ 
		$RS' = \{\}$ \;
		$RS'' = \{\}$
		\For{$Z \in RS$}{
			Linear transformation: $Z' = \mathrm{lintrans}(Z)$ (see Equation \ref{eq:lintrans}) \;
			$RS' = RS' \cup \{Z'\}$ \;
		}
		\If{over approximate} {
			\For{$Z \in RS'$}{
				Apply ReLU activation function: $RS'' = RS'' \cup \mathrm{RsO}(Z, \mathrm{MaxAmp})$  \;
			}
		}
		\If{under approximate} {
			\For{$Z \in RS'$}{
				Apply ReLU activation function: $RS'' = RS'' \cup \mathrm{RsU}(Z, \mathrm{MaxAmp})$ \;
			}
		}
		$RS = RS''$ \;
		\If{$|RS| \geq $ MaxZono}{
			Finde smallest zonotopes $\in RS$ \;
			Remove smallest zonotopes from $RS$ \;
			\If{over approximate}{
				Union smallest zonotopes over approximativly by interval hull\;
				Add union to $RS$\;
			}
		}
	}
	\Return{$RS$}\;
	\caption{{\sc PropZLimit} propagates zonotope through ReLU network (limited no. subsets)}
	\label{algo:propoagate_nn_limit}
\end{algorithm}

\newpage
\subsection{Extension of RsO and RsU for Large(r) Neural Networks}
\label{sec:larger_nn}

The subsection ``Reachable Set Approximation: Analysis of the Limit’’ (page \pageref{sec:limits}) shows that the number of zonotopes required to approximate the reachable set might increase exponentially with the number of neurons of the neural network in the worst case. Thus, we propose an extension (see page \pageref{sec:balancing}: "Balancing approximation tightness and run-time") that allows to restrict the number of total subsets (max. zono.) and the amplification (max. amp.). The maximum amplification is the maximum number of subsets used to approximate ReLU(Z) w.r.t. the zonotope Z subjected to ReLU, while the maximum number of zonotopes is the maximum number of zonotopes w.r.t. to the whole neural network that is used to approximate the reachable set. 
This restriction allows to use RsO for robustness verification of larger neural networks. 
To illustrate how these restrictions affect the tightness of our approximations, we compare the performance of RsO for different max. zono. and max. amp. values on a neural network with 5 hidden layers of $30$ neurons on the FashionMNIST data set (see Figure~\ref{fig:res_extension}). 
Without limitations, RsO might require up to $2^{150} \approx 1.43 \cdot 10^{45}$ subsets. 

\begin{figure}[ht]
	\centering
	\includegraphics[width=0.49\textwidth]{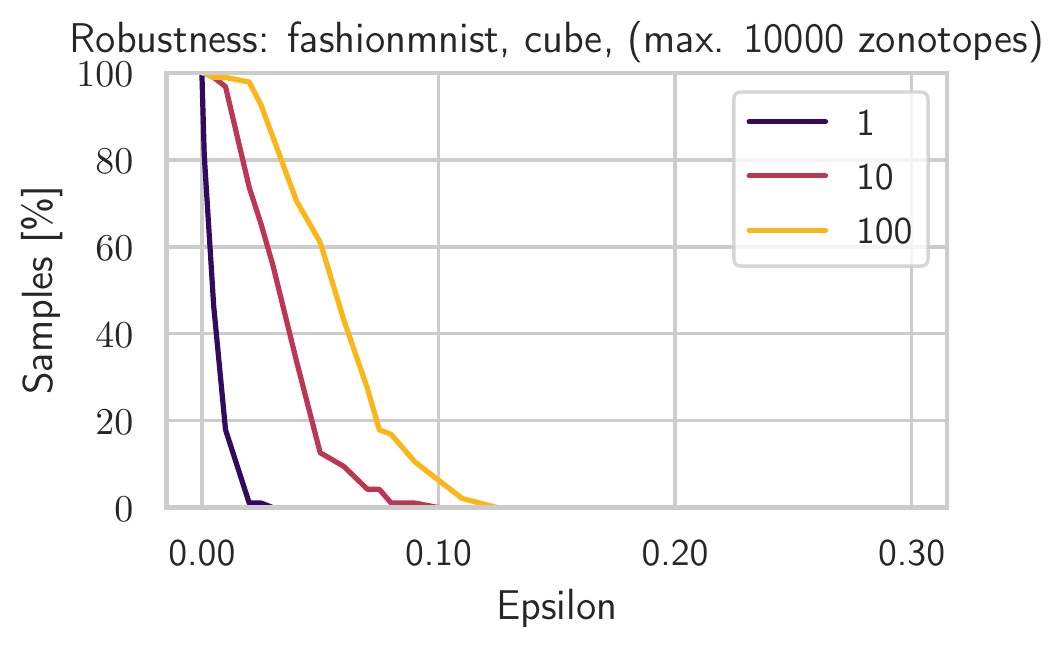}
	\includegraphics[width=0.49\textwidth]{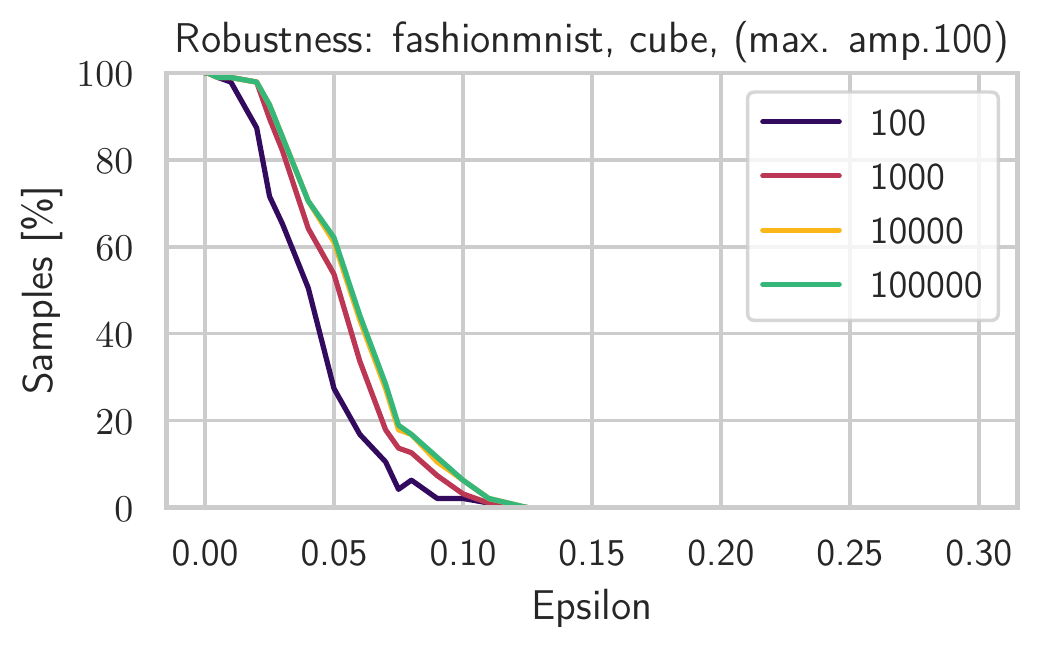}
	\caption{Performance of RsO on robustness verification using fashionmnist (classification acc. $95\%$) with different max. amp. and fixed max. zono. (left), with fixed max. amp. and different max. zono. (right).}
	\label{fig:res_extension}
\end{figure}

Figure~\ref{fig:res_extension} illustrates that the number of robustness certificates increases with the maximum amplification (left plot). Furthermore, the number of robustness certificates increases with max. zono. up to $10,000$, but choosing larger max. zono. does not result in a further increase of robustness certificates (right plot). 
Thus, to obtain tight approximations the max. amp. should be chosen as large as possible and feasible, while the max. zono. should be chosen as small as possible but as large as necessary to obtain the maximum performance.

\end{document}